\documentclass[pdflatex,sn-mathphys-num]{sn-jnl}

\usepackage{graphicx}%
\usepackage{multirow}%
\usepackage{amsmath,amssymb,amsfonts}%
\usepackage{amsthm}%
\usepackage{mathrsfs}%
\usepackage[title]{appendix}%
\usepackage{xcolor}%
\usepackage{textcomp}%
\usepackage{manyfoot}%
\usepackage{booktabs}%
\usepackage{algorithm}%
\usepackage{algorithmicx}%
\usepackage{algpseudocode}%
\usepackage{listings}%
\usepackage{setspace}%
\usepackage{geometry}%
\usepackage{float}%
\usepackage{afterpage}%
\usepackage{subcaption}
\usepackage{caption}

\theoremstyle{thmstyleone}%
\newtheorem{theorem}{Theorem}
\newtheorem{lemma}{Lemma}

\theoremstyle{thmstyletwo}%
\newtheorem{remark}{Remark}%

\theoremstyle{thmstylethree}%
\newtheorem{definition}{Definition}%

\begin{document}

\title[Properties of Approximations for ...]{Convergence Analysis of Max-Min Exponential Neural Network Operators in Orlicz Space}


\author[1]{\fnm{Satyaranjan } \sur{Pradhan}}\email{satya.math1993@gmail.com}

\author*[1,]{\fnm{Madan Mohan} \sur{Soren}}\email{mms.math@buodisha.edu.in}


\affil[1]{\orgdiv{Department of Mathematics}, \orgname{Berhampur University}, \orgaddress{\street{Berhampur}, \city{Ganjam}, \postcode{760007}, \state{Odisha}, \country{India}}}




\abstract{In this current work, we propose a Max–Min approach for approximating functions using exponential neural network operators. We extend this framework to develop the Max–Min Kantorovich-type exponential neural network operators and investigate their approximation properties. We study both pointwise and uniform convergence for univariate functions. To analyze the order of convergence, we use the logarithmic modulus of continuity and estimate the corresponding rate of convergence. Furthermore, we examine the convergence behavior of the Max–Min Kantorovich-type exponential neural network operators within the Orlicz space setting. We provide some graphical representations to illustrate the approximation error of the function through suitable kernel and sigmoidal activation functions.}




\maketitle

\section{Introduction}\label{sec1}

The fundamental process of function sampling and reconstruction occurs at the center of approximation theory due to its wide applications in the field of signal analysis and image processing \cite{bede2016}. In 1949, a significant result was introduced by Shanon together with Whittekar and Kotelnikov, known as the \textbf{\bfseries WKS sampling theorem}, which establishes that band-limited signals can be completely reconstructed from their uniformly spaced sample nodes. However, in real-world scenarios, most signals are not strictly band-limited \cite{slepian1976bandwidth}. To address this, Butzer and Stens provided a constructive approach to approximate such signals over the real line\cite{butzer1992sampling}.
In parallel to classical sampling frameworks, the concept of exponential sampling was developed by researchers like Ostrowsky et al. \cite{ostrowsky1981}, Bertero and Pike \cite{bertero1991}, and Gori \cite{gori1993}, focusing on the reconstruction of functions from data sampled at exponentially spaced nodes. Mamedov was the first to initiate the study of the Mellin transform \cite{Mamedov1991}, and the foundational mathematical structure of exponential sampling was rigorously established by Butzer and Jansche, who demonstrated the advantages of Mellin transform methods over Fourier-based approaches in such settings\cite{Butzer1998,Butzer1997}.The generalized exponential sampling series is defined as
\begin{equation}
  (S^\varphi_w h)(z) := \sum_{k=-\infty}^{\infty} \varphi(e^{-k} z^{w}) \,\, h\left(e^{\frac{k}{w}}\right), \quad z \in \mathbb{R}_{+},\, k \in \mathbb{Z},\, w > 0,   
\end{equation}
where the series is absolutely convergent for suitable functions \( h : \mathbb{R}^+ \to \mathbb{R} \). 
These operators have been studied extensively for their convergence properties and approximation potential in various functional settings \cite{bardaro2017, bardaro2019b, balsamo2019}. In particular, Bardaro et al. \cite{bardaro2017} developed a generalization of the exponential sampling series and analyzed their approximation properties in detail, while their subsequent work \cite{bardaro2019b} examined convergence in Mellin–Lebesgue spaces. More recent advancements have extended the classical theory by incorporating Kantorovich-type modifications, where the sampled function is replaced by its local integral mean.The Kantorovich exponential sampling series is given by
\begin{equation}
    (I^\varphi_w h)(z) := \sum_{k=-\infty}^{\infty} \left( w \int_{\frac{k}{w}}^{\frac{k+1}{w}} h(e^u) \, du \right) \varphi(e^{-k} z^{w}),
\end{equation}
which converges absolutely for any locally integrable function \( h : \mathbb{R}_{+} \to \mathbb{R} \).
Such Kantorovich versions have been developed in univariate and multivariate settings and studied with a variety of kernels, including Hadamard-type and fractional integral types \cite{Acar2023a, Acar2023b, Kursun2023}. Acar and Kursun \cite{Acar2023a} investigated the pointwise convergence of  Kantorovich exponential sampling series, while their work with Turgay \cite{Acar2023b} introduced multidimensional Kantorovich modifications. Similarly, Kursun and Aral \cite{Kursun2023} explored Hadamard-type kernel modifications, and Angamuthu and Bajpeyi \cite{angamuthu2020} provided both direct and inverse approximation results for such operators. In addition to purely linear operators, exponential sampling techniques have been integrated into the framework of neural networks. Exponential sampling based neural network operators were introduced by Bajpeyi and Kumar \cite{bajpeyi2021b}.

Neural Network (NN) operators have played a pivotal role in approximation theory since the foundational work of Cybenko \cite{Cybenko1989} and Hornik et al. \cite{Hornik1989}, which established that feedforward neural networks with sigmoidal activation functions are universal approximators for continuous functions on compact subsets of $\mathbb{R}^n$. A feedforward neural network with a single hidden layer can be represented as :
\[
N_n(z) = \sum_{k=0}^{n} c_k \, \sigma(w_k \cdot z - \theta_k), \quad z = (z_1, \dots, z_d) \in \mathbb{R}^d,
\]
where \( n \in \mathbb{N} \) denotes the number of neurons in the hidden layer, \( \theta_k \in \mathbb{R} \) are the threshold values, and \( w_k = (w_1, \dots, w_d) \in \mathbb{R}^d \) are the weights associated with each neuron. The coefficients \( c_k \in \mathbb{R} \) represent the connection strengths of the output neuron, and \( \sigma : \mathbb{R} \to \mathbb{R} \) is the activation function. The notation \( w_k \cdot z \) refers to the standard inner product in \( \mathbb{R}^d \).

 Building on this foundation, researchers such as Costarelli and Vinti developed rigorous approximation frameworks using NN operators that involve sigmoidal and other smooth activation functions \cite{costarelli2014, costarelli2016a}, including their behavior in Orlicz spaces and for Lebesgue integrable functions.

An alternative and powerful approach in the analysis of  NN operators was introduced in \cite{cardaliaguet1992} as an extension of the framework developed in \cite{anastassiou2010}, employing the \textit{max-product method}. Originally initiated by Coroianu and Gal \cite{coroianu2010,coroianu2011,coroianu2012}, the max-product technique involves replacing summation symbols in finite sums or series with the supremum operator. This transformation turns linear operators into nonlinear ones, often resulting in enhanced approximation capabilities, especially when approximating continuous functions. Further developments include the convergence of max-product versions of classical Kantorovich-type neural network operators studied by Costarelli and Vinti in \cite{costarelli2016b}.

More recent studies have extended these results to broader functional settings. For instance, Boccali, Costarelli, and Vinti have extended these operators to Orlicz spaces, providing modular convergence results and quantitative estimates across \(L^p\), interpolation, and exponential spaces\cite{boccali2023,boccali2023a}. In the multivariate setting, Angamuthu has investigated the approximation behavior of multivariate max-product Kantorovich exponential sampling operators, establishing results concerning pointwise and uniform convergence, along with rates based on logarithmic moduli of smoothness\cite{angamuthu2023}. Jin, Yu, and Zhou have analyzed truncated versions of these operators in \(L^p\) spaces, deriving Jackson-type inequalities and employing maximal function techniques to quantify approximation errors\cite{jin2023}. Cantarini et al. have addressed the inverse approximation and saturation properties, establishing optimal approximation orders for the Kantorovich-type max-product neural network operators\cite{cantarini2022}. Furthermore, Pradhan and Soren have studied the behavior of these operators in weighted function spaces, deriving convergence estimates and a quantitative Voronovskaja-type theorem that characterizes their asymptotic performance~\cite{pradhan2025}. Recently in  \cite{bajpeyi2021b} Bajpeyi and Angamuthu introduced the exponial sampling type neural network operators and  Bajpeyi et al.  \cite{Bajpeyi2024} extended this framework to the exponential sampling type Kantorovich max-product NN operators, defined as :
\begin{equation}
(T_m^{\varrho} \zeta)(z) := 
\frac{
\displaystyle \bigvee_{i \in J^m} \varphi(e^{-i} z^m) \cdot m \int_{i/m}^{(i+1)/m} \zeta (e^u) \, du
}{
\displaystyle \bigvee_{i \in J^m} \varphi(e^{-i} z^m)
}, \quad z \in \mathbb{R}_+,
\end{equation}\label{eq1}
where \( I = [e^{-1}, e] \), \( \zeta : I \to \mathbb{R}_+ \) is a locally integrable function, \( \varphi \) is a sigmoidal function acting as the activation function, and \( J^m := \{ i \in \mathbb{Z} : i = -m, \dots, m-1 \} \).

In this paper, we are focusing on the conjugation of exponential sampling with neural network-based max--min operators. The exploration of max--min operators has received significant attention, especially in fuzzy approximation theory. Gökçer and Duman \cite{gokcer2020} established a general framework for max-min operators, proving their effectiveness in approximating nonnegative continuous functions and providing error estimates for Hölder continuous functions. Extending this, Gökçer and Aslan \cite{gokcer2022a} studied Kantorovich-type max-min operators, presenting quantitative convergence theorems and their applications to integral equations and signal processing, thus enabling the approximation of a broader class of Lebesgue integrable functions. Moreover, regular summability techniques were introduced to improve approximation behavior \cite{gokcer2022b}. Most recently, Aslan \cite{aslan2025} proposed max-min NN operators and analyzed their approximation properties in depth, further strengthening the theoretical foundations and application potential of the max-min approach.

This paper is structured as follows: In Section \ref{sec2}, we begin with the core ideas and fundamental concept of exponential sampling theory along with neural network-based approximation operators. We introduce in section \ref{sec3},  the max–min Kind exponential sampling neural network operator $ \mathfrak{GM}_{n} $ and analyze its convergence properties, both  pointwise and uniform. We  further derive the approximation rate  using the logarithmic modulus of continuity. Section \ref{sec4} extends the analysis of Section \ref{sec3} to the Max-Min Kantorovich type exponential sampling neural network operator $ \mathfrak{MK}_{n}^{(m)} $ including the study of convergence behaviour and rate of approximation via the logarithmic modulus of continuity. The order of convergence is analyzed under the assumption that the target function exhibits log-Hölder continuity of order $ \tau $, and, in addition convergence results are established in the setting of Orlicz space. Finally, in Section \ref{sec5}, numerical examples are provided to support the theoretical results and illustrate the approximation efficiency of the proposed operators.

\section{Preliminaries}\label{sec2}
In this paper, we use notations $\mathbb{N},\, \mathbb{Z},\, \mathbb{R},\,  \mathbb{R_{+}},\, \mathbb{R}_{+}^{0}$ to denote the sets of positive integers,\, integers,\, real numbers,\, positive real numbers,\, nonnegative real numbers, respectively. Through out this work we consider the interval $\mathscr{I} = [a,b]$ be any compact subset of $\mathbb{R}_{+}$.
\begin{definition}\label{def1}
    A function $\mathscr{F}:\mathbb{R}_{+} \rightarrow \mathbb{R}$ is defined to be  $\log$ continuous at a point $z \in\mathbb{R}_{+} $ if for each given $\epsilon > 0$,  $ \exists\,\rho > 0 $ such that $\,\, \forall y \in \mathbb{R}_{+} $, \[|\log(y) -\log(z)| \leq \rho \quad \Longrightarrow  \quad |\mathscr{F}(y) - \mathscr{F}(z)| < \epsilon .\]
\end{definition}
\begin{definition}\label{def1}
    A function $\mathscr{F}:\mathbb{R}_{+} \rightarrow \mathbb{R}$ is defined to be $\log$ uniformly continuous function if for any given $\epsilon > 0$, $ \exists\,\rho > 0 $ such that  $\quad \forall z_{1},z_{2} \in \mathbb{R}_{+}$, 
    \[|\log(z_{1}) -\log(z_{2})| \leq \rho \quad \Longrightarrow  \quad |\mathscr{F}(z_{1}) - \mathscr{F}(z_{2})| < \epsilon .\]
\end{definition}
Now, we introduce the space
\begin{itemize}[label={}]
    \item  $\mathscr{U}_{b}(\mathscr{I})\, := \, \{ \mathscr{F}: \mathscr{I} \to \mathbb{R}\, \vert  \text{$\mathscr{F}$ is  a bounded, $\log$ uniformly continuous function}\}$ equipped with the supremum norm $\|\mathscr{F}\|_{\infty} \,:= \,\sup\limits_{z\in \mathscr{I}} |\mathscr{F}(z)|.$
    \item $\mathcal{L}^{\infty}(\mathscr{I}) := \{\mathscr{F}: \mathscr{I} \to \mathbb{R}\,\, \vert\,     \operatorname{ess\,sup}\limits_{z\in\mathscr{I}} |\mathscr{F}|< \infty \}$
\end{itemize}
It is evident that log uniformly continuity  and uniformly continuity are equivallent over any compact subset of $\mathbb{R}_{+}.$
\subsection{ Mellin Orlicz space}
We define a version of the Haar measure $ \mu $ on $ \mathbb{R}_+ $ by:
\[\mu(H) = \int_H \frac{dv}{v}, \quad \text{for each measurable subset } H \subset \mathbb{R}_+.\] 
We denote by  $ \mathrm{M}(\mathbb{R}_+, \mu) $, the space of all measurable functions on \( \mathbb{R}_+ \) relative to the Haar measure \( \mu \).\\
Let $\eta : \mathbb{R}_+^0 \to \mathbb{R}_+^0 $ be a \emph{convex \(\varphi\)-function}, i.e., a function satisfying:
\setlength{\itemsep}{1pt}
\setlength{\parskip}{0pt}
\begin{description}
    \item[(i)\,\,:] $\eta$ is a non-decreasing and continuous function on $\mathbb{R}_0^+$;
    \item[(ii)\,:] $\eta(0) = 0$ and $\eta(u) > 0$ for all $u > 0$;
    \item[(iii):] $\eta$ is convex on $\mathbb{R}_0^+$.
\end{description}

We define the associated \textbf{modular functional} \( I_\eta \) corresponding to the Haar measure by:
\begin{align}\label{MFdef}
    I_\eta[\mathscr{F}] \,:= \int_0^\infty \eta(\,|\mathscr{F}(z)|\,) \,\, \frac{dz}{z}, \quad \mathscr{F} \in \mathrm{M}(\mathbb{R}_+, \mu).
\end{align}
It is evident that the  functional $I_\eta$ is convex. Next, we define the Orlicz space over \( \mathbb{R}_+ \) with respect to \( \mu \) and \( \eta \) as:
\[
L^\eta_\mu(\mathbb{R}_+) := \left\{ \mathscr{F}\in \mathrm{M}(\mathbb{R}_+, \mu) \,:\, \exists \, \ell > 0 \text{ such that } I_\eta[\ell \mathscr{F}] < \infty \right\}.\]
For fixed \( c \in \mathbb{R} \), the \textbf{Mellin-Orlicz space} is thereby defined as follows:
\[X_c^\eta := \left\{ \mathscr{F} : \mathbb{R}_+ \to \mathbb{R} \ \middle| \ \mathscr{F}(\cdot) \cdot (\cdot)^c \in L^\eta_\mu(\mathbb{R}_+) \right\}.\]
In this paper, we focus on the Mellin-Orlicz space for \( c = 0 \), denoted by:
$X_0^\eta = L^\eta_\mu(\mathbb{R}_+),$ which is simply the Orlicz space over \( \mathbb{R}_+ \) with Haar measure \( \mu \).

Now we define a subspace of finite elements of \( L^\eta_\mu(\mathbb{R}_+) \) as:
\[\mathscr{E}^\eta_\mu(\mathbb{R}_+) := \left\{ \mathscr{F} \in \mathrm{M}(\mathbb{R}_+, \mu) : I_\eta[\ell \mathscr{F}] < \infty \text{ for all } \ell > 0 \right\}.\]
A net $ (\mathscr{F}_m)_{m > 0}\, \subset \, L^\eta_\mu(\mathbb{R}_+) $ is said to \textbf{modularly converge} to $ \mathscr{F} \,\in\, L^\eta_\mu(\mathbb{R}_+) $ \\ if $\exists \, \,\ell > 0 \text{ such that }:$
\[ \lim_{m \to \infty} I_\eta[\ell\,(\mathscr{F}_m - \mathscr{F})] = 0.\]
For given \( \mathscr{F} \in L^\eta_\mu(\mathbb{R}_+) \), the Luxemburg norm is defined as:
\[\|\mathscr{F}\|_\eta := \inf \left\{ \ell > 0 : I_\eta\left[\frac{\mathscr{F}}{\ell}\right] \leq 1 \right\}.\]

We say that the convex $\varphi-$ function $\eta$ satisfies the \textbf{\( \Delta_2 \)-condition} if 
\[ \exists \quad\text{a constant} \,M >0\,\, \text{such that}\,\,\eta(2u) \leq M\, \eta(u), \quad \forall u \geq 0.\]

This condition ensures the the the the equivalence between norm and modular convergence. Specifically, if \( \eta \) satisfies the \( \Delta_2 \)-condition, then:
\[\|\mathscr{F}_m - \mathscr{F}\|_\eta \to 0 \iff I_\eta[\ell\,(\mathscr{F}_m - \mathscr{F})] \to 0 \quad \text{(for some } \ell > 0).\]

Moreover, under the \( \Delta_2 \)-condition, we have\,:\quad $L^\eta_\mu(\mathbb{R}_+) = \mathscr{E}^\eta_\mu(\mathbb{R}_+).$\\
The aforementioned theory is now presented through examples of convex $\varphi$- functions, highlighting key aspects of Mellin Orlicz space, as follows:
\subsubsection*{Example-1}

Let \( p > 1 \), and define:\, $\eta(u) = u^p, \quad u \geq 0.$ Then the functional \( I_\eta[f] \) becomes:\quad$I_\eta[\mathscr{F}] = \int\limits_0^\infty |\mathscr{F}(z)|^p \, \frac{dz}{z},$
and the associated Orlicz space \( L^\eta_\mu(\mathbb{R}_+) \) coincides with the  \( L^p \)-space over Haar measure $\mu$ is $L^p_\mu(\mathbb{R}_+)$. In this case, the  corresponding Mellin-Orlicz space becomes:
\[X_c^p := \left\{ \mathscr{F} \in L^p_\mu(\mathbb{R}_+): \int_0^\infty |\mathscr{F}(z)|^p z^{cp - 1} \, dz < \infty \right\}.\]
Moreover, the $\Delta_2$ condition on $\eta(u)$ ensures the equivalence between Luxemburg norm convergence and modular convergence.

\subsubsection*{Example-2}
For $\alpha >0 $, define $\eta_{\alpha}(u) = \exp{(u^{\alpha})} -1,\quad \forall \,u\geq 0$. Then the functional $I_{\eta_{\alpha}}[f] $  becomes: \, 
$I_{\eta_{\alpha}}[\mathscr{F}] = \int\limits_{0}^{\infty} \left[ \exp\left( |\mathscr{F}(z)|^{\alpha} \right) - 1 \right] \, \frac{dz}{z}, \quad \mathscr{F} \in \mathrm{M}(\mathbb{R}_{+},\mu).$\\ The associated Orlicz space over the Haar measure $\mu$ is given by:
\[L_{\mu}^{\eta_{\alpha}}(\mathbb{R}_{+}) := \left\{ \mathscr{F} \in \mathrm{M}(\mathbb{R}_{+}, \mu) \ \bigg| \ \int\limits_{0}^{\infty} \left[ \exp({|\mathscr{F}(z)|^{\alpha})} - 1 \right] \frac{dz}{z} < \infty \right\}.\]
In the Mellin setting, this becomes:
\[X_{c}^{\eta_{\alpha}} := \left\{ \mathscr{F} \, \in L_{\mu}^{\eta_{\alpha}}(\mathbb{R}_{+})\ : \ \int\limits_{0}^{\infty} \left[ \exp({|\mathscr{F}(z)|^{\alpha})} - 1 \right] \, z^{\,c-1} \, dz < \infty \right\}.\]

In this case, The  equivalence between Luxemberg norm convergence and Modular convergence does not hold as $\eta_{\alpha}$ fails to satisfy $\Delta_{2} $ condition.

\subsection{Sigmoidal Activation and Centered Bell Kernel}
\begin{definition}\label{def2}
    A measurable function $\delta : \mathbb{R} \to \mathbb{R}$ is defined  to be a sigmoidal activation function if 
   \begin{equation*}
        \lim_{z\to -\infty} \delta(z) = 0 \quad \And \quad \lim_{z\to \infty} \delta(z) = 1.
   \end{equation*}
\end{definition}

Let $\mathcal{C}^{(2)}(\mathbb{R})$ be the space of all twice continuously differentiable functions in $\mathbb{R}$. Assume that $\delta \in \mathcal{C}^{(2)}(\mathbb{R})$ is a nondecreasing and satisfies the conditions stated below:
\begin{enumerate}
    \item[]$\Delta{1}$~: $\delta(z) - \frac{1}{2}$ is an odd function.
    \item[]$\Delta{2}$~: $\delta(z)$ is concave on $\mathbb{R_{+}}.$
    \item[]$\Delta{3}$~: $\delta(z)= \mathcal{O}\left(|z|^{-1-v}\right)$, as $z \to -\infty$, for some $v> 0$
\end{enumerate}
\begin{remark}\label{rm1}
    From $\Delta{1}$, we can conclude that $\delta(z) + \delta(-z) = 1 \quad \forall \, z \in \mathbb{R}_{+}$ and using the fact  $\delta$ is nondecreasing together with condition $\Delta{2}$, we obtain $\delta^{'}(z) \geq 0, \quad \forall \, z \geq 0.$
\end{remark}
For the NN operator of exponential type, we define  the density function $\Psi_{\delta}$, using the sigmoidal  activation function $\delta $  as the kernel, as follows :
\begin{equation*}
    \Psi_{\delta}(z)= \frac{1}{2}\left[\delta(\log{z} +\,1) - \delta(\log{z} -1)\right], \quad z\in \mathbb{R_{+}}.
\end{equation*}
Since $\delta$ is non decreasing, so $\Psi_{\delta}$ is non negative on $\mathbb{R_{+}}$. Moreover, with aditional assumption $\delta(2) > \delta(0)$, it follows that $\Psi_{\delta} (e) >0$.

 Let  $\Psi_{\delta}$  be the Kernel function satisfying the properties stated below :
\begin{lemma}\label{lm1} (see \cite{costarelli2013a})
    \begin{enumerate}
        \item $\Psi_{\delta}(z) \geq \, \, 0 \quad  \forall \, z \in {\mathbb{R}_{+}}.$ 
        \item $\Psi_{\delta}(z)$ is even.
        \item $\underset{z\to  \infty}{\lim} \Psi_{\delta}(z) \, = \, \lim\limits_{z\to {0^{+}}} \Psi_{\delta}(z) = 0 .$\\[2pt]
        \item $ \Psi_{\delta}(z) $ is nondecreasing on $(0,1)$ and non increasing on $(1,\infty).$
        \item  $ \sum\limits_{k\in \mathbb{Z}} \Psi_{\delta}(e^{-k}z) = 1,\quad \forall z\in \mathbb{R}_{+}.$
    \end{enumerate}
\end{lemma}

\begin{remark}\label{rem2}
According to the definition provided by Cardaliaguet and Euvrard in \cite{cardaliaguet1992}, the function \( \Psi_\delta(z) \) is a \emph{centered bell-shaped function}.
\end{remark}

\begin{remark}
It is important to note that if we omit condition \( (\Delta 2) \), and instead directly assume that the kernel \( \Psi_\delta \) satisfies assertion 4 of Lemma-\ref{lm1}, the main theoretical results of the framework remain valid. For further discussion, see for example \cite{cardaliaguet1992}.
\end{remark}
\noindent
To establish  the well-definiteness of the operator in the Max-Min setting  we need the follwing lemma.

\begin{lemma}\label{lm2}(see\cite{bajpeyi2021b})
    For a fixed $z \in \mathbb{R_{+}}$, we have
    \[\bigvee_{k\in \mathbb{Z}} \Psi_{\delta}(e^{-k}{}z^{n}) \geq \Psi_{\delta}(e) >0,\] for $\bar{k} \in \mathbb{Z}\,$ such that $z\in \left[e^{\frac{\bar{k}}{n}}, e^{\frac{\bar{k}+1}{n}}\right)\,$ and $n\in \mathbb{N}.$
\end{lemma}

\begin{lemma}\label{lm3}(see \cite{bajpeyi2021b})
    Let the sigmoidal activation function $\delta(z)\,$ satisfies condition $(\Delta{3})$ for $v > l $, for some $l \in \mathbb{N}$.Then for every fixed $ \gamma > 0 $, we have:\[ \lim_{n\to \infty} \bigvee_{|\log{u} -\, k| \geq \,\gamma n} \Psi_{\delta}(e^{-k} u)  |\log{}u- k|^{q} = 0,\] uniformly with respect to $u \in \mathbb{R}_{+}$, for every $0\leq q \leq l.$

\end{lemma}

\begin{lemma}\label{lm4}(see \cite{bajpeyi2021c})
    For $ v $ in $(\Delta3)$and $\gamma >0$, we have :
    \[ \bigvee_{|\log{u} - k|\, \geq\, \gamma n} \Psi_{\delta}(e^{-k} u) = \mathcal{O}(n^{-v})\,\, \text{as}\,\,  n \to \infty,\]  
    uniformly with respect to $u \in \mathbb{R}_{+}.$
\end{lemma}
\begin{lemma}\label{lm5} (see \cite{aslan2025})                                                        If $\bigvee\limits_{m \in \mathbb{Z}}x_{m} < \infty \, \text{or}\, \bigvee\limits_{m \in \mathbb{Z}}y_{m} < \infty $, then we have: 
\[\bigvee\limits_{m \in \mathbb{Z}}x_{m}  - \bigvee\limits_{m \in \mathbb{Z}}y_{m} \leq \bigvee\limits_{m \in \mathbb{Z}}\left|x_{m}-y_{m} \right| .\]                                                           
\end{lemma}

\begin{lemma}\label{lm6} (see \cite{bede2016}) 
    For any $q,r,s, \in [0,1]$, we have:
    \[\left|q \land r - q \land s\right|\leq q \land \left| r-s\right|, \]
    where $ q \land r \, $  denotes the min$\{q,r\}.$
\end{lemma}

\begin{lemma}\label{lm7}(see \cite{bede2008})
    For any $r,s,t \geq 0$, we have \[ r\land t + s \land t \geq  (r+s)\land t .\]
\end{lemma}

\begin{lemma}\label{lm8}(see\cite{bajpeyi2025})
   Let  $ I $ be a finite index set and $ \lambda > 0 $. Then for sequences $ \{x_m\}_{m \in I} $ and $ \{y_m\}_{m \in I} $ with values in $ [0,1] $, we have
\[
\lambda \bigvee_{m \in I} (x_m \wedge y_m) = \bigvee_{m \in I} (\lambda x_m \wedge \lambda y_m).
\]
 
\end{lemma}
\begin{definition}\label{def3}
    For a given $z\in \mathscr{I}\subseteq \mathbb{R}_{+}$ and $ n \in \mathbb{N}$ and $\gamma > 0$,\\ we define $\mathcal{B}_{\gamma, n} $ as  : 
    \[\mathscr{B}_{\gamma, n} : = \{ p = \lceil{n\log{a}}\rceil,.......\lfloor{n\log{b}}\rfloor  : \, \left| \frac{p}{n}- \log{z}\right| \leq \gamma \}.\]
\end{definition}
 
\section{Results of max-min exponential neural network operator}\label{sec3}
In this section, we summarize the key properties of Max-Min exponential NN operator, including  its pointwise and uniform convergence, and estimates its convergence rate  using logarithmic modulus of smoothness for $\mathfrak{GM}_{n}.$\\[2pt]
 Let $\mathscr{F}: \mathscr{I} \to [0,1] $ be a log uniformly continuous and bounded function. Now we define the Max-Min kind of exponential neural network operator  $\mathfrak{GM}_{n}(\mathscr{F};z)$  as follows:
 \begin{equation}\label{Meq1}
     \mathfrak{GM}_{n}(\mathscr{F};z) := \bigvee\limits_{k=\lceil{n\log{}a}\rceil}^{\lfloor{n\log{}b}\rfloor} \mathscr{F} (e^{\frac{k}{n}})\,  \land \, \frac{\Psi_{\delta}(e^{-k}z^{n})}{\bigvee\limits_{k=\lceil{n\log{}a}\rceil}^{\lfloor{n\log{}b}\rfloor}\Psi_{\delta}(e^{-k}z^{n})}    ,   \left(\forall z \in [a,b]\right).
 \end{equation}
For sufficiently large $n \in \mathbb{N}$, we always obtain $\lceil{n\log{a}}\rceil \leq \lfloor{n\log{b}}\rfloor$ and whenever $\lceil{n\log{a}}\rceil \leq p \leq  \lfloor{n\log{b}}\rfloor$, it follows that  $ a\leq e^{\frac{p}{n}} \leq b $. Further by combining lemma \ref{lm2}, and the boundedness of the function $\mathscr{F}$, we conclude that the aforementioned operators are well defined.
 \par 
Some important properties of the Max-Min exponential neural network operator are given in the following lemma.
 \begin{lemma}\label{lm8}
     For two bounded functions $\mathscr{F}_{1}, \mathscr{F}_{2} : \mathscr{I} \to [0,1]$, the following properties hold.
     \begin{enumerate}
         \item[(a)] $\mathfrak{GM}_{n}(\mathscr{F}_{1})$ is log continuous on  $\mathscr{I}$.
         \item[(b)] If $\mathscr{F}_{1}(z) \,\leq \,\mathscr{F}_{2}(z)\,\,\,  \text{then}\,\,\, \mathfrak{GM_{n}}(\mathscr{F}_{1}) \leq\mathfrak{GM_{n}}(\mathscr{F}_{2})\,\, \forall z \in \mathscr{I}.$
         \item[(c)] For  sufficiently large $n\in \mathbb{N}$ and  $ \forall \,z \in \mathbb{R_{+}}$, we have:\[ \left|\mathfrak{GM}_{n}(\mathscr{F}_{1};z) - \mathfrak{GM}_{n}(\mathscr{F}_{2};z) \right| \leq \mathfrak{GM}_{n}(|\mathscr{F}_{1} - \mathscr{F}_{2}|;z).\]
         \item[(d)] $\mathfrak{GM}_{n}(\mathscr{F}_{1} + \mathscr{F}_{2};z)\leq \mathfrak{GM}_{n}(\mathscr{F}_{1};z)+ \mathfrak{GM}_{n}(\mathscr{F}_{2};z)\,\, \, ,  \forall z \in \mathbb{R}_{+}$.
         \item[(e)] $\mathfrak{GM}_{n}$ is Pseudo linear.\\
         i.e $ \mathfrak{GM}_{n}\left(\alpha \land \mathscr{F}_{1}) \bigvee ( \beta \land  \mathscr{F}_{2}) \, ;z\right) = \alpha \land \mathfrak{GM}_{n}(\mathscr{F}_{1} ;z) \bigvee  \, \beta \land \mathfrak{GM}_{n}( \mathscr{F}_{2};z).$
     \end{enumerate}
\end{lemma}
\begin{proof}
    The proof of (a),(b) is obvious.Now for the proof of (c) we have:\\[0.2em]
    $\left| \mathfrak{GM}_{n}(\mathscr{F}_{1};z) - \mathfrak{GM}_{n}(\mathscr{F}_{2};z)\right|$
    \begin{align*}
          & = \left|\bigvee\limits_{k=\lceil{n\log{}a}\rceil}^{\lfloor{n\log{}b}\rfloor} \mathscr{F}_{1} (e^{\frac{k}{n}})\,  \land \, \frac{\Psi_{\delta}(e^{-k}z^{n})}{\bigvee\limits_{k=\lceil{n\log{}a}\rceil}^{\lfloor{n\log{}b}\rfloor}\Psi_{\delta}(e^{-k}z^{n})} - \bigvee\limits_{k=\lceil{n\log{}a}\rceil}^{\lfloor{n\log{}b}\rfloor} \mathscr{F}_{2} (e^{\frac{k}{n}})\,  \land \, \frac{\Psi_{\delta}(e^{-k}z^{n})}{\bigvee\limits_{k=\lceil{n\log{}a}\rceil}^{\lfloor{n\log{}b}\rfloor}\Psi_{\delta}(e^{-k}z^{n})}\right| \\
          & \leq  \bigvee\limits_{k=\lceil{n\log{}a}\rceil}^{\lfloor{n\log{}b}\rfloor} \left|\mathscr{F}_{1} (e^{\frac{k}{n}})\,  \land \, \frac{\Psi_{\delta}(e^{-k}z^{n})}{\bigvee\limits_{k=\lceil{n\log{}a}\rceil}^{\lfloor{n\log{}b}\rfloor}\Psi_{\delta}(e^{-k}z^{n})}-  \mathscr{F}_{2} (e^{\frac{k}{n}})\,  \land \, \frac{\Psi_{\delta}(e^{-k}z^{n})}{\bigvee\limits_{k=\lceil{n\log{}a}\rceil}^{\lfloor{n\log{}b}\rfloor}\Psi_{\delta}(e^{-k}z^{n})}\right|.\\
    \end{align*}
    By using the lemma~\ref{lm6}, we get :\\
    \\
 $\left| \mathfrak{GM}_{n}(\mathscr{F}_{1};z) - \mathfrak{GM}_{n}(\mathscr{F}_{2};z)\right|$
    \begin{align*}
          & \leq \bigvee\limits_{k=\lceil{n\log{}a}\rceil}^{\lfloor{n\log{}b}\rfloor} \left|\mathscr{F}_{1} (e^{\frac{k}{n}})- \mathscr{F}_{2} (e^{\frac{k}{n}})\right|\,  \land \, \frac{\Psi_{\delta}(e^{-k}z^{n})}{\bigvee\limits_{k=\lceil{n\log{}a}\rceil}^{\lfloor{n\log{}b}\rfloor}\Psi_{\delta}(e^{-k}z^{n})} \\
          & = \mathfrak{GM}_{n}(|\mathscr{F}_{1} - \mathscr{F}_{2}|;z).
    \end{align*}
Now in the proof of part(d), using the Lemma~\ref{lm7}, we have:
    \begin{align*}
        \mathfrak{GM}_{n}(\mathscr{F}_{1} + \mathscr{F}_{2};z)= & \bigvee\limits_{k=\lceil{n\log{}a}\rceil}^{\lfloor{n\log{}b}\rfloor} (\mathscr{F}_{1} +\mathscr{F}_{2})(e^{\frac{k}{n}})\,  \land \, \frac{\Psi_{\delta}(e^{-k}z^{n})}{\bigvee\limits_{k=\lceil{n\log{}a}\rceil}^{\lfloor{n\log{}b}\rfloor}\Psi_{\delta}(e^{-k}z^{n})}\\
        \leq & \bigvee\limits_{k=\lceil{n\log{}a}\rceil}^{\lfloor{n\log{}b}\rfloor} \mathscr{F}_{1} (e^{\frac{k}{n}})\,  \land \, \frac{\Psi_{\delta}(e^{-k}z^{n})}{\bigvee\limits_{k=\lceil{n\log{}a}\rceil}^{\lfloor{n\log{}b}\rfloor}\Psi_{\delta}(e^{-k}z^{n})} \\ & + \bigvee\limits_{k=\lceil{n\log{}a}\rceil}^{\lfloor{n\log{}b}\rfloor} \mathscr{F}_{2} (e^{\frac{k}{n}})\,  \land \, \frac{\Psi_{\delta}(e^{-k}z^{n})}{\bigvee\limits_{k=\lceil{n\log{}a}\rceil}^{\lfloor{n\log{}b}\rfloor}\Psi_{\delta}(e^{-k}z^{n})}\\
        = & \mathfrak{GM}_{n}(\mathscr{F}_{1};z) + \mathfrak{GM}_{n}(\mathscr{F}_{2};z).
    \end{align*}
    For(e) let $z\in \mathscr{I} $ be fixed and $ \alpha ,\, \beta \in [0,1]$.Then since $([0,1],\lor, \land )$ is an ordered semiring, there holds
    \par
    $\mathfrak{GM}_{n}\left( (\alpha \land \mathscr{F}_{1})\bigvee (\beta \land \mathscr{F}_{2}) \, ;z \right)$
    \begin{align*}
        = & \bigvee\limits_{k=\lceil{n\log{}a}\rceil}^{\lfloor{n\log{}b}\rfloor} \left( \alpha \land \mathscr{F}_{1} (e^{\frac{k}{n}}) \bigvee \beta \land \mathscr{F}_{2} (e^{\frac{k}{n}})\right)\,  \land \, \frac{\Psi_{\delta}(e^{-k}z^{n})}{\bigvee\limits_{k=\lceil{n\log{}a}\rceil}^{\lfloor{n\log{}b}\rfloor}\Psi_{\delta}(e^{-k}z^{n})}\\
        = & \bigvee\limits_{k=\lceil{n\log{}a}\rceil}^{\lfloor{n\log{}b}\rfloor} \left( \alpha \land \mathscr{F}_{1} (e^{\frac{k}{n}}) \right)\,  \land \, \frac{\Psi_{\delta}(e^{-k}z^{n})}{\bigvee\limits_{k=\lceil{n\log{}a}\rceil}^{\lfloor{n\log{}b}\rfloor}\Psi_{\delta}(e^{-k}z^{n})} \\ &  \bigvee \,  \left( \beta \land \mathscr{F}_{2} (e^{\frac{k}{n}})\right)\,  \land \, \frac{\Psi_{\delta}(e^{-k}z^{n})}{\bigvee\limits_{k=\lceil{n\log{}a}\rceil}^{\lfloor{n\log{}b}\rfloor}\Psi_{\delta}(e^{-k}z^{n})}\\
        = & \, \alpha \land \bigvee\limits_{k=\lceil{n\log{}a}\rceil}^{\lfloor{n\log{}b}\rfloor}   \mathscr{F}_{1} (e^{\frac{k}{n}}) \,  \land \, \frac{\Psi_{\delta}(e^{-k}z^{n})}{\bigvee\limits_{k=\lceil{n\log{}a}\rceil}^{\lfloor{n\log{}b}\rfloor}\Psi_{\delta}(e^{-k}z^{n})} \\ & \bigvee\beta \land  \bigvee\limits_{k=\lceil{n\log{}a}\rceil}^{\lfloor{n\log{}b}\rfloor} \,   \mathscr{F}_{2} (e^{\frac{k}{n}})\,  \land \, \frac{\Psi_{\delta}(e^{-k}z^{n})}{\bigvee\limits_{k=\lceil{n\log{}a}\rceil}^{\lfloor{n\log{}b}\rfloor}\Psi_{\delta}(e^{-k}z^{n})}\\
        = & \alpha \land \mathfrak{GM}_{n}(\mathscr{F}_{1}\, ; z)\, \, \bigvee  \beta \land \mathfrak{GM}_{n}(\mathscr{F}_{2}\, ; z).
    \end{align*}
    for  sufficiently large $n \in \mathbb{N}$.
\end{proof}
\vspace{1em}
\subsection{Pointwise \& Uniform Convergence}\label{subsec2}
\begin{theorem}
    Let $\mathscr{F}:\mathscr{I} \to [0,1]\,$ be given. Then at any point of log continuity $z\in \mathscr{I} \subseteq \mathbb{R}_{+}$ \[ \lim\limits_{n\to \infty} \mathfrak{GM}_{n}(\mathscr{F};z) = \mathscr{F}(z) .\]
Moreover if $\mathscr{F} \in \mathcal{U}_{b}(\mathscr{I})$, then the approximation is uniform i.e
\[\lim\limits_{n\to \infty} \left||\mathfrak{GM}_{n}(\mathscr{F};z) - \mathscr{F}(z)\right|| = 0 ,\]
where $\| .\|$ denotes the supremum norm.
\end{theorem}

\begin{proof}
    Suppose the map $\mathscr{F}:\mathscr{I} \to [0,1]\,$ is log continuous at $z\in \mathscr{I}$. By  triangle inequality,\\ 
    \par
    $\left|\mathfrak{GM}_{n}(\mathscr{F};z) - \mathscr{F}(z)\right|$
    \begin{align*}
        = & \left| \bigvee\limits_{k=\lceil{n\log{a}}\rceil}^{\lfloor{n\log{}b}\rfloor} \mathscr{F}(e^{\frac{k}{n}}) \,  \land  \frac{\Psi_{\delta}(e^{-k}z^{n})}{\bigvee\limits_{k=\lceil{n\log{}a}\rceil}^{\lfloor{n\log{}b}\rfloor}\Psi_{\delta}(e^{-k}z^{n})} - \mathscr{F}(z)\right|\\
        \leq & \left| \bigvee\limits_{k=\lceil{n\log{}a}\rceil}^{\lfloor{n\log{}b}\rfloor} \mathscr{F}(e^{\frac{k}{n}}) \,  \land  \frac{\Psi_{\delta}(e^{-k}z^{n})}{\bigvee\limits_{k=\lceil{n\log{}a}\rceil}^{\lfloor{n\log{}b}\rfloor}\Psi_{\delta}(e^{-k}z^{n})} -\bigvee\limits_{k=\lceil{n\log{}a}\rceil}^{\lfloor{n\log{}b}\rfloor} \mathscr{F}(z) \,  \land  \frac{\Psi_{\delta}(e^{-k}z^{n})}{\bigvee\limits_{k=\lceil{n\log{}a}\rceil}^{\lfloor{n\log{}b}\rfloor}\Psi_{\delta}(e^{-k}z^{n})}\right|\\&+ \left|\bigvee\limits_{k=\lceil{n\log{}a}\rceil}^{\lfloor{n\log{}b}\rfloor} \mathscr{F}(z) \,  \land  \frac{\Psi_{\delta}(e^{-k}z^{n})}{\bigvee\limits_{k=\lceil{n\log{}a}\rceil}^{\lfloor{n\log{}b}\rfloor}\Psi_{\delta}(e^{-k}z^{n})}  - \mathscr{F}(z)\right|\\ 
        = & \mathbb{I}_{1} + \mathbb{I}_{2}.
    \end{align*}
      To estimate $\mathbb{I}_{1}$, using lemma \ref{lm5}  and \ref{lm6} we write:
     \begin{align*}
         \mathbb{I}_{1}= & \left| \bigvee\limits_{k=\lceil{n\log{}a}\rceil}^{\lfloor{n\log{}b}\rfloor} \mathscr{F}(e^{\frac{k}{n}}) \,  \land  \frac{\Psi_{\delta}(e^{-k}z^{n})}{\bigvee\limits_{k=\lceil{n\log{}a}\rceil}^{\lfloor{n\log{}b}\rfloor}\Psi_{\delta}(e^{-k}z^{n})} -\bigvee\limits_{k=\lceil{n\log{}a}\rceil}^{\lfloor{n\log{}b}\rfloor} \mathscr{F}(z) \,  \land  \frac{\Psi_{\delta}(e^{-k}z^{n})}{\bigvee\limits_{k=\lceil{n\log{}a}\rceil}^{\lfloor{n\log{}b}\rfloor}\Psi_{\delta}(e^{-k}z^{n})}\right|\\
         \leq &  \bigvee\limits_{k=\lceil{n\log{a}}\rceil}^{\lfloor{n\log{b}}\rfloor} \left| \mathscr{F}(e^{\frac{k}{n}}) - \mathscr{F}(z) \right| \,  \land  \frac{\Psi_{\delta}(e^{-k}z^{n})}{\bigvee\limits_{k=\lceil{n\log{}a}\rceil}^{\lfloor{n\log{}b}\rfloor}\Psi_{\delta}(e^{-k}z^{n})} .
     \end{align*}
     Now splitting the Max operation over the index set:
     \begin{align*}
         \mathbb{I}_{1} \leq & \bigvee\limits_{k\in\mathcal{B}_{\gamma, n}} \left| \mathscr{F}(e^{\frac{k}{n}}) - \mathscr{F}(z) \right| \,  \land  \frac{\Psi_{\delta}(e^{-k}z^{n})}{\bigvee\limits_{k=\lceil{n\log{}a}\rceil}^{\lfloor{n\log{}b}\rfloor}\Psi_{\delta}(e^{-k}z^{n})} \\ &  \bigvee \, \bigvee\limits_{k\notin \mathcal{B}_{\gamma, n}} \left| \mathscr{F}(e^{\frac{k}{n}}) - \mathscr{F}(z) \right| \,  \land  \frac{\Psi_{\delta}(e^{-k}z^{n})}{\bigvee\limits_{k=\lceil{n\log{}a}\rceil}^{\lfloor{n\log{}b}\rfloor}\Psi_{\delta}(e^{-k}z^{n})} \\
         \leq & \mathbb{I}_{1.1} \bigvee  \mathbb{I}_{1.2} .
     \end{align*}
    Since $\mathscr{F}(z)$ is log continuous at z , then  for every $ \varepsilon >0,\, \mathbb{I}_{1.1}$ can be written as
    \begin{align*}
        \mathbb{I}_{1.1} \leq & \bigvee\limits_{k\in\mathcal{B}_{\gamma, n}} \left| \mathscr{F}(e^{\frac{k}{n}}) - \mathscr{F}(z) \right| \,  \land  \frac{\Psi_{\delta}(e^{-k}z^{n})}{\bigvee\limits_{k=\lceil{n\log{}a}\rceil}^{\lfloor{n\log{}b}\rfloor}\Psi_{\delta}(e^{-k}z^{n})}\\
        \leq & \bigvee\limits_{k\in\mathcal{B}_{\gamma, n}} \epsilon \,  \land  \frac{\Psi_{\delta}(e^{-k}z^{n})}{\bigvee\limits_{k=\lceil{n\log{}a}\rceil}^{\lfloor{n\log{}b}\rfloor}\Psi_{\delta}(e^{-k}z^{n})}
        \leq\bigvee\limits_{k\in\mathcal{B}_{\gamma, n}} \epsilon \,  \land 1\
        \leq  \epsilon.   
    \end{align*}
     For $\mathbb{I}_{1.2}$, using $a\land b \leq b\,\,, \forall a,b \in \mathbb{R}{+}$, we have: 
     \begin{align*}
         \mathbb{I}_{1.2} = & \bigvee\limits_{k\notin \mathcal{B}_{\gamma, n}} \left| \mathscr{F}(e^{\frac{k}{n}}) - \mathscr{F}(z) \right| \,  \land  \frac{\Psi_{\delta}(e^{-k}z^{n})}{\bigvee\limits_{k=\lceil{n\log{}a}\rceil}^{\lfloor{n\log{}b}\rfloor}\Psi_{\delta}(e^{-k}z^{n})}\\
         \leq & \bigvee\limits_{k\notin \mathcal{B}_{\gamma, n}}  \frac{\Psi_{\delta}(e^{-k}z^{n})}{\bigvee\limits_{k=\lceil{n\log{}a}\rceil}^{\lfloor{n\log{}b}\rfloor}\Psi_{\delta}(e^{-k}z^{n})}
         \leq\frac{1}{\Psi_{\delta}(e)} \bigvee\limits_{k\notin \mathcal{B}_{\gamma, n}}\Psi_{\delta}(e^{-k}z^{n})
         \leq \frac{c\,n^{-v} }{\Psi_{\delta}(e)}
         \leq  \epsilon.
     \end{align*}
     for sufficiently large $n\in \mathbb{N}$, where $v$ corresponds to given $ v \, $in $\Delta3$.\\[2pt]
    From $\mathbb{I}_{2}$, we have: 
    \begin{align*}
        \mathbb{I}_{2} = & \,\left|\bigvee\limits_{k=\lceil{n\log{}a}\rceil}^{\lfloor{n\log{}b}\rfloor} \mathscr{F}(z) \,  \land  \frac{\Psi_{\delta}(e^{-k}z^{n})}{\bigvee\limits_{k=\lceil{n\log{}a}\rceil}^{\lfloor{n\log{}b}\rfloor}\Psi_{\delta}(e^{-k}z^{n})}  - \mathscr{F}(z)\right|\\ 
        = & \,\left|\bigvee\limits_{k=\lceil{n\log{}a}\rceil}^{\lfloor{n\log{}b}\rfloor} (\mathscr{F}(z) \land 1) \,  \land  \frac{\Psi_{\delta}(e^{-k}z^{n})}{\bigvee\limits_{k=\lceil{n\log{}a}\rceil}^{\lfloor{n\log{}b}\rfloor}\Psi_{\delta}(e^{-k}z^{n})}  - \mathscr{F}(z)\right|\\ 
        \leq & \,\left| \mathscr{F}(z)\, \land \bigvee\limits_{k=\lceil{n\log{}a}\rceil}^{\lfloor{n\log{}b}\rfloor}  1 \,  \land  \frac{\Psi_{\delta}(e^{-k}z^{n})}{\bigvee\limits_{k=\lceil{n\log{}a}\rceil}^{\lfloor{n\log{}b}\rfloor}\Psi_{\delta}(e^{-k}z^{n})}  - \mathscr{F}(z)\right|\\ 
        \leq & \, \left| \mathscr{F}(z)\, \land \bigvee\limits_{k=\lceil{n\log{}a}\rceil}^{\lfloor{n\log{}b}\rfloor}   \frac{\Psi_{\delta}(e^{-k}z^{n})}{\bigvee\limits_{k=\lceil{n\log{}a}\rceil}^{\lfloor{n\log{}b}\rfloor}\Psi_{\delta}(e^{-k}z^{n})}  - \mathscr{F}(z)\right|\\ 
        \leq &\, \left| \mathscr{F}(z)\, \land 1  - \mathscr{F}(z)\right|\\ 
        = &\, 0.
    \end{align*}
Hence,  $\left|\mathfrak{GM}_{n}(\mathscr{F};z) - \mathscr{F}(z)\right| \, \leq \varepsilon$, and since $\varepsilon$ is arbitrary, the result follows.
\end{proof}

\subsection{Order Of Convergence}\label{subsec3}

In this section, we'll use the logarithmic Modulus of smoothness for estimating the order of convergence of the operators $\mathfrak{GM}_{n}$.
\begin{definition}\label{def4}(see \cite{bajpeyi2021c})
   For $\mathscr{F} \in \mathcal{U}_{b}(\mathscr{I})$, the logarithmic modulus of smoothness is defined as : 
   \[{\Omega}_{\mathscr{I}}(\mathscr{F},\rho):=\sup\limits_{s,t \in \mathscr{I}} \{\left|\mathscr{F}(s) - \mathscr{F}(t)\right| : \,  \text{whenever} \| \log{}s - \log{}t\| \leq \rho, \, \rho >0 \}.\]
   ${\Omega}_{\mathscr{I}}(\mathscr{F},\rho)$ satisfies the following properties:
   \begin{enumerate}
       \item[(a)] $\lim\limits_{\rho \to 0}{\Omega}_{\mathscr{I}}(\mathscr{F},\rho) \to 0 $.
       \item[(b)] $ \text{For all } \beta > 0 : \quad {\Omega}_{\mathscr{I}}(\mathscr{F},\beta\rho)\leq (\beta +1)\, {\Omega}_{\mathscr{I}}(\mathscr{F},\rho)$.
       \item[(c)] $ \text{For all} \, s,t \in \mathscr{I} :\quad\left|\mathscr{F}(s) - \mathscr{F}(t)\right| \leq {\Omega}_{\mathscr{I}}(\mathscr{F},\rho) \left( 1 + \frac{\| \log{s} - \log{t}\|}{\rho}\right).$
   \end{enumerate}
\end{definition}

\begin{definition}
    Let \( j \in \mathbb{N} \), with \( j > 0 \). The \emph{generalized absolute moment} of order \( j \) of the function \( \Psi_{\delta} \) is given by
    \[\mathfrak{A}_j(\Psi_{\delta}) := \sup_{u \in \mathbb{R}^+} \bigvee_{k \,\in \,\mathbb{Z}} \left| \Psi_{\delta}(e^{-k}\, u) \right| \left| k \,- \,\log u \right|^j.\]
\end{definition}

\begin{theorem}\label{thm2}
    Let $\mathscr{F}:\mathscr{I} \to [0,1]$,and let $\rho_{n}$ be a null sequence with  $\frac{1}{n\rho_{n}} \to 0 $ as $n \to \infty$ and $v >0$ satisfies the condition $\Delta3$. then for any $\mathscr{F} \in \mathcal{U}_{b}^{+}(\mathscr{I})$, we have 
    \[ \|\mathfrak{GM}_{n}(\mathscr{F})- \mathscr{F}\| \leq \Omega_{\mathscr{I}}(\mathscr{F},\rho_{n}) \bigvee  \, \frac{\mathfrak{A}_{v}(\Psi_{\delta})}{\Psi_{\delta}(e)\, n^{v}(\rho_{n})^{v}}. \]
\end{theorem}

\begin{proof}
    From~ \ref{Meq1}, by applying the traingle inequality and splitting the maximum operation over the index set, we obtain:
    \begin{align*}
        \|\mathfrak{GM}_{n}(\mathscr{F})- \mathscr{F} \| \leq &\bigvee\limits_{k=\lceil{n\log{}a}\rceil}^{\lfloor{n\log{}b}\rfloor}\left| \mathscr{F}(e^{\frac{k}{n}}) - \mathscr{F}(z) \right| \,  \land  \frac{\Psi_{\delta}(e^{-k}z^{n})}{\bigvee\limits_{k=\lceil{n\log{}a}\rceil}^{\lfloor{n\log{}b}\rfloor}\Psi_{\delta}(e^{-k}z^{n})}\\
        \leq &\bigvee\limits_{k\in\mathcal{B}_{\gamma, n}} \left| \mathscr{F}(e^{\frac{k}{n}}) - \mathscr{F}(z) \right| \,  \land  \frac{\Psi_{\delta}(e^{-k}z^{n})}{\bigvee\limits_{k=\lceil{n\log{}a}\rceil}^{\lfloor{n\log{}b}\rfloor}\Psi_{\delta}(e^{-k}z^{n})} \\ &  \bigvee \, \bigvee\limits_{k\notin \mathcal{B}_{\gamma, n}} \left| \mathscr{F}(e^{\frac{k}{n}}) - \mathscr{F}(z) \right| \,  \land  \frac{\Psi_{\delta}(e^{-k}z^{n})}{\bigvee\limits_{k=\lceil{n\log{}a}\rceil}^{\lfloor{n\log{}b}\rfloor}\Psi_{\delta}(e^{-k}z^{n})} \\
        =: &\, \mathbb{R}_{1} \bigvee  \mathbb{R}_{2}.
    \end{align*}
    For the  $\,\mathbb{R}_{1}$, using definition \ref{def4} we obtain:\\
    \begin{align*}
        \mathbb{R}_{1} \leq & \bigvee\limits_{k=\lceil{n\log{a}}\rceil}^{\lfloor{n\log{b}}\rfloor}\Omega_{\mathscr{I}}\left( \mathscr{F}, \left| e^{\frac{k}{n}}-z\right|\right) \,  \land  \frac{\Psi_{\delta}(e^{-k}z^{n})}{\bigvee\limits_{k=\lceil{n\log{a}}\rceil}^{\lfloor{n\log{b}}\rfloor}\Psi_{\delta}(e^{-k}z^{n})} \\
        \leq & \bigvee\limits_{k=\lceil{n\log{a}}\rceil}^{\lfloor{n\log{b}}\rfloor}\Omega_{\mathscr{I}}\left( \mathscr{F}, \left| \frac{k}{n}-\log{}z\right|\right) \,  \land  \frac{\Psi_{\delta}(e^{-k}z^{n})}{\bigvee\limits_{k=\lceil{n\log{}a}\rceil}^{\lfloor{n\log{}b}\rfloor}\Psi_{\delta}(e^{-k}z^{n})} \\
        \leq & \bigvee\limits_{k=\lceil{n\log{}a}\rceil}^{\lfloor{n\log{}b}\rfloor}\Omega_{\mathscr{I}}\left( \mathscr{F}, \left| \frac{k}{n}-\log{}z\right|\right) \\
        \leq & \Omega_{\mathscr{I}}\left( \mathscr{F}, \rho_{n}\right).
    \end{align*}
    Now considering $\mathbb{R}_{2}$, we have:\\
    \begin{align*}
        \mathbb{R}_{2} \leq & \bigvee\limits_{k\notin \mathcal{B}_{\gamma, n}} \left| \mathscr{F}(e^{\frac{k}{n}}) - \mathscr{F}(z) \right| \,  \land  \frac{\Psi_{\delta}(e^{-k}z^{n})}{\bigvee\limits_{k=\lceil{n\log{}a}\rceil}^{\lfloor{n\log{}b}\rfloor}\Psi_{\delta}(e^{-k}z^{n})} \\
        \leq & \bigvee\limits_{k\notin \mathcal{B}_{\gamma, n}}  \frac{\Psi_{\delta}(e^{-k}z^{n})}{\bigvee\limits_{k=\lceil{n\log{}a}\rceil}^{\lfloor{n\log{}b}\rfloor}\Psi_{\delta}(e^{-k}z^{n})} \\
        \leq & \frac{1}{\Psi_{\delta}(e)} \, \bigvee\limits_{\left| n \log{}z -k\right| > n\rho_{n}} \Psi_{\delta}(e^{-k}z^{n}).
    \end{align*}
    From definition \ref{def3} we can conclude that $\frac{\left| n \log{}z -k\right|^{v}}{n^{v}\, (\rho_{n})^{v}} > 1 $, where $v$ corresponds to $(\Delta3)$.\\
    Using the above fact, we obtain: 
    \begin{align*}
        \mathbb{R}_{2} < & \frac{1}{\Psi_{\delta}(e)\, n^{v}\, (\rho_{n})^{v}} \bigvee\limits_{\left| n \log{}z -k\right| > n\rho_{n}} \Psi_{\delta}(e^{-k}z^{n}) {\left| n \log{}z -k\right|^{v}}\\
        < & \frac{1}{\Psi_{\delta}(e)\, n^{v}\, (\rho_{n})^{v}} \bigvee\limits_{k \in \mathbb{Z}} \Psi_{\delta}(e^{-k}z^{n}) {\left| n \log{}z -k\right|^{v}}\\
        = & \frac{1}{\Psi_{\delta}(e)\, n^{v}\, (\rho_{n})^{v}} \mathfrak{A}_{v}(\Psi_{\delta}).
    \end{align*}
    which is desired.
\end{proof}

\section{Results of Kantorovich-type max-min exponential neural network operator}\label{sec4}
In this section, we define  the Kantorovich type Max-Min exponential NN operator  along with its key properties, and derive its pointwise and uniform convergence, and estimates its convergence rate  via logarithmic modulus of smoothness. Furthermore, convergence results are established in the setting of orlicz space.
\subsection{Definition}\label{kantdef1}
Let \( \mathscr{F}: [a, b] \to \mathbb{R} \) be a locally integrable function and \( \Psi_\delta \) be a positive kernel function parameterized by \( \delta > 0 \). The \textbf{Kantorovich-type exponential max--min neural network operator} is defined as
\[
\mathfrak{MK}_n^{(m)}(\mathscr{F}; z) := \bigvee\limits_{k = \lceil n \log a \rceil}^{\lfloor n \log b \rfloor} \left[ n\int\limits_{\frac{k}{n}}^{\frac{k+1}{n}} \mathscr{F}(e^{u})\,du \ \land \ \frac{\Psi_\delta(e^{-k} z^n)}{\bigvee\limits_{j = \lceil n \log a \rceil}^{\lfloor n \log b \rfloor} \Psi_\delta(e^{-j} z^n)} \right], \quad \forall z \in [a, b],
\] 
where \( \land \) and \( \bigvee \) denote the minimum and maximum operations, respectively. Also, it can be written as 
\[
\mathfrak{MK}_n^{(m)}(\mathscr{F}; z) := \bigvee\limits_{k \in \mathscr{J}_{n}} \left[ n\int\limits_{\frac{k}{n}}^{\frac{k+1}{n}} \mathscr{F}(e^{u})\,du \ \land \ \frac{\Psi_\delta(n\log z -k)}{\bigvee\limits_{j\in \mathscr{J}_{n}} \Psi_\delta(n\log z -j)} \right], \quad \forall z \in [a, b],\]  where $\mathscr{J}_{n} = \lceil{n\log a\rceil}\dots \lfloor{n\log b\rfloor}$

\begin{lemma}\label{lm1kant}
 For two bounded functions $\mathscr{F}_{1}, \mathscr{F}_{2} : \mathscr{I} \to [0,1]$, the following properties hold.
     \begin{enumerate}
         \item[(a)] $\mathfrak{MK}_{n}^{(m)}(\mathscr{F}_{1})$ is log continuous on  $\mathscr{I}$  .
         \item[(b)] If $\mathscr{F}_{1}(z) \leq \mathscr{F}_{2}(z)\, \forall z \in \mathscr{I}$ then $ \mathfrak{MK_{n}^{(m)}}(\mathscr{F}_{1}) \leq\mathfrak{MK_{n}^{(m)}}(\mathscr{F}_{2})\,\, \forall z \in \mathscr{I}.$
         \item[(c)] For all sufficiently large $n\in \mathbb{N}$ and  $ \forall \,z \in \mathbb{R_{+}}$ we have:\[ \left|\mathfrak{MK}_{n}^{(m)}(\mathscr{F}_{1};z) - \mathfrak{MK}_{n}^{(m)}(\mathscr{F}_{2};z) \right| \leq \mathfrak{MK}_{n}^{(m)}(|\mathscr{F}_{1} - \mathscr{F}_{2}|;z).\]
         \item[(d)] $\mathfrak{MK}_{n}^{(m)}(\mathscr{F}_{1} + \mathscr{F}_{2};z)\leq \mathfrak{MK}_{n}^{(m)}(\mathscr{F}_{1};z)+ \mathfrak{MK}_{n}^{(m)}(\mathscr{F}_{2};z)\,\, \, ,  \forall z \in \mathbb{R}_{+}$.
        
     \end{enumerate}
\end{lemma}

\subsection{Pointwise \& Uniform Convergence}\label{subsec2}
\begin{theorem}\label{Kantthm1}
let $\mathscr{F}:\mathscr{I} \to [0,1]\,$ be given  . Then at any point of log continuity $z\in \mathscr{I} \subseteq \mathbb{R}_{+}$ \[ \lim\limits_{n\to \infty} \mathfrak{MK}_{n}^{(m)}(\mathscr{F};z) = \mathscr{F}(z). \]
Moreover, if $\mathscr{F} \in \mathcal{U}_{b}(\mathscr{I})$, then the approximation is uniform i.e
\[\lim\limits_{n\to \infty} \left\|\mathfrak{MK}_{n}^{(m)}(\mathscr{F}\,;\,z) - \mathscr{F}(z)\right\| = 0 ,\]
where $\| .\|$ denotes the supremum norm.
\end{theorem}
\begin{proof}
    let $\mathscr{F}:\mathscr{I} \to [0,1]\,$ be a function such that $\mathscr{F}$ is log continuous at $z\in \mathscr{I}$.Then by  triangle inequality\\ 
    \par
    $\left|\mathfrak{MK}_{n}^{(m)}(\mathscr{F};z) - \mathscr{F}(z)\right|$
    \begin{align*}
        = &\, \left| \bigvee\limits_{k=\lceil{n\log{}a}\rceil}^{\lfloor{n\log{}b}\rfloor} n \int\limits_{\frac{k}{n}}^{\frac{k+1}{n}}\mathscr{F}(e^{u}) du \,  \land  \frac{\Psi_{\delta}(e^{-k}z^{n})}{\bigvee\limits_{k=\lceil{n\log{}a}\rceil}^{\lfloor{n\log{}b}\rfloor}\Psi_{\delta}(e^{-k}z^{n})} - \mathscr{F}(z)\right|\\
        \leq &\, \left| \bigvee\limits_{k=\lceil{n\log{}a}\rceil}^{\lfloor{n\log{}b}\rfloor} n \int\limits_{\frac{k}{n}}^{\frac{k+1}{n}} \mathscr{F}(e^{u}) du  \,  \land  \frac{\Psi_{\delta}(e^{-k}z^{n})}{\bigvee\limits_{k=\lceil{n\log{}a}\rceil}^{\lfloor{n\log{}b}\rfloor}\Psi_{\delta}(e^{-k}z^{n})} \right. \\ & \left. \quad -\bigvee\limits_{k=\lceil{n\log{}a}\rceil}^{\lfloor{n\log{}b}\rfloor} n \int\limits_{\frac{k}{n}}^{\frac{k+1}{n}} \mathscr{F}(z) du  \,  \land  \frac{\Psi_{\delta}(e^{-k}z^{n})}{\bigvee\limits_{k=\lceil{n\log{}a}\rceil}^{\lfloor{n\log{}b}\rfloor}\Psi_{\delta}(e^{-k}z^{n})}\right|\\ &+ \left|\bigvee\limits_{k=\lceil{n\log{}a}\rceil}^{\lfloor{n\log{}b}\rfloor} n \int\limits_{\frac{k}{n}}^{\frac{k+1}{n}}\mathscr{F}(z) du  \,  \land  \frac{\Psi_{\delta}(e^{-k}z^{n})}{\bigvee\limits_{k=\lceil{n\log{}a}\rceil}^{\lfloor{n\log{}b}\rfloor}\Psi_{\delta}(e^{-k}z^{n})}  - \mathscr{F}(z)\right|\\ 
        = & \quad\mathbb{I}_{1} + \mathbb{I}_{2}.
    \end{align*}
     To estimate $\mathbb{I}_{1}$, using lemma \ref{lm5}  and \ref{lm6} we write:
   \begin{align*}
       \mathbb{I}_{1} = & \left| \bigvee\limits_{k=\lceil{n\log{}a}\rceil}^{\lfloor{n\log{}b}\rfloor} n \int\limits_{\frac{k}{n}}^{\frac{k+1}{n}} \mathscr{F}(e^{u}) du  \,  \land  \frac{\Psi_{\delta}(e^{-k}z^{n})}{\bigvee\limits_{k=\lceil{n\log{}a}\rceil}^{\lfloor{n\log{}b}\rfloor}\Psi_{\delta}(e^{-k}z^{n})} \right. \\ & \left. \quad -\bigvee\limits_{k=\lceil{n\log{}a}\rceil}^{\lfloor{n\log{}b}\rfloor} n \int\limits_{\frac{k}{n}}^{\frac{k+1}{n}} \mathscr{F}(z) du  \,  \land  \frac{\Psi_{\delta}(e^{-k}z^{n})}{\bigvee\limits_{k=\lceil{n\log{}a}\rceil}^{\lfloor{n\log{}b}\rfloor}\Psi_{\delta}(e^{-k}z^{n})}\right|\\ 
       \leq &  \bigvee\limits_{k=\lceil{n\log{}a}\rceil}^{\lfloor{n\log{}b}\rfloor} \left|n \int_{\frac{k}{n}}^{\frac{k+1}{n}} \left(\mathscr{F}(e^{u}) - \mathscr{F}(z)\right)\,du \right| \,  \land  \frac{\Psi_{\delta}(e^{-k}z^{n})}{\bigvee\limits_{k=\lceil{n\log{}a}\rceil}^{\lfloor{n\log{}b}\rfloor}\Psi_{\delta}(e^{-k}z^{n})} \\
         \leq & \bigvee\limits_{k=\lceil{n\log{}a}\rceil}^{\lfloor{n\log{}b}\rfloor} n \int_{\frac{k}{n}}^{\frac{k+1}{n}} \left|\left(\mathscr{F}(e^{u}) - \mathscr{F}(z)\right|\right)\,du  \,  \land  \frac{\Psi_{\delta}(e^{-k}z^{n})}{\bigvee\limits_{k=\lceil{n\log{}a}\rceil}^{\lfloor{n\log{}b}\rfloor}\Psi_{\delta}(e^{-k}z^{n})}
     \end{align*}
     Now splitting  the Max operation over the index set:
     \begin{align*}
         \mathbb{I}_{1} \leq & \bigvee\limits_{k\in\mathcal{B}_{\gamma, n}}  n \int_{\frac{k}{n}}^{\frac{k+1}{n}} \left|\left(\mathscr{F}(e^{u}) - \mathscr{F}(z)\right|\right)\,du \,  \land  \frac{\Psi_{\delta}(e^{-k}z^{n})}{\bigvee\limits_{k=\lceil{n\log{}a}\rceil}^{\lfloor{n\log{}b}\rfloor}\Psi_{\delta}(e^{-k}z^{n})} \\ &  \bigvee \, \bigvee\limits_{k\notin \mathcal{B}_{\gamma, n}}  n \int_{\frac{k}{n}}^{\frac{k+1}{n}} \left|\left(\mathscr{F}(e^{u}) - \mathscr{F}(z)\right|\right)\,du\,  \land  \frac{\Psi_{\delta}(e^{-k}z^{n})}{\bigvee\limits_{k=\lceil{n\log{}a}\rceil}^{\lfloor{n\log{}b}\rfloor}\Psi_{\delta}(e^{-k}z^{n})} \\
         \leq & \mathbb{I}_{1.1} \bigvee  \mathbb{I}_{1.2}.
     \end{align*}
    Since $\mathscr{F}(z)$ is log continuous at z , then  for every $ \, \varepsilon > 0,\,\,\mathbb{I}_{1.1}$ can be written as:
    \begin{align*}
        \mathbb{I}_{1.1} \leq & \bigvee\limits_{k\in\mathcal{B}_{\gamma, n}} n \int_{\frac{k}{n}}^{\frac{k+1}{n}} \left|\left(\mathscr{F}(e^{u}) - \mathscr{F}(z)\right|\right)\,du  \,  \land  \frac{\Psi_{\delta}(e^{-k}z^{n})}{\bigvee\limits_{k=\lceil{n\log{}a}\rceil}^{\lfloor{n\log{}b}\rfloor}\Psi_{\delta}(e^{-k}z^{n})}\\
        \leq & \bigvee\limits_{k\in\mathcal{B}_{\gamma, n}} \epsilon \,  \land  \frac{\Psi_{\delta}(e^{-k}z^{n})}{\bigvee\limits_{k=\lceil{n\log{}a}\rceil}^{\lfloor{n\log{}b}\rfloor}\Psi_{\delta}(e^{-k}z^{n})}
        \leq \bigvee\limits_{k\in\mathcal{B}_{\gamma, n}} \epsilon \,  \land 1
        \leq  \epsilon.   
    \end{align*}
    \vspace{0.4em}
     For $\mathbb{I}_{1.2}$, using  $a\land b \leq b\,\,, \forall a,b \in \mathbb{R}{+}$:
     \begin{align*}
         \mathbb{I}_{1.2} = & \bigvee\limits_{k\notin \mathcal{B}_{\gamma, n}} n \int_{\frac{k}{n}}^{\frac{k+1}{n}} \left|\left(\mathscr{F}(e^{u}) - \mathscr{F}(z)\right|\right)\,du  \,  \land  \frac{\Psi_{\delta}(e^{-k}z^{n})}{\bigvee\limits_{k=\lceil{n\log{}a}\rceil}^{\lfloor{n\log{}b}\rfloor}\Psi_{\delta}(e^{-k}z^{n})}\\
         \leq & \bigvee\limits_{k\notin \mathcal{B}_{\gamma, n}}  \frac{\Psi_{\delta}(e^{-k}z^{n})}{\bigvee\limits_{k=\lceil{n\log{}a}\rceil}^{\lfloor{n\log{}b}\rfloor}\Psi_{\delta}(e^{-k}z^{n})}
         \leq  \frac{1}{\Psi_{\delta}(e)} \bigvee\limits_{k\notin \mathcal{B}_{\gamma, n}}\Psi_{\delta}(e^{-k}z^{n})
         \leq  \frac{c\,n^{-v} }{\Psi_{\delta}(e)}
         \leq  \epsilon.
     \end{align*}
     for sufficiently large $n\in \mathbb{N}$, where $v$ corresponds to given $ v \, $in $\Delta3$.
    From $\mathbb{I}_{2}$, 
    \begin{align*}
        \mathbb{I}_{2} = & \left|\bigvee\limits_{k=\lceil{n\log{}a}\rceil}^{\lfloor{n\log{}b}\rfloor} n \int_{\frac{k}{n}}^{\frac{k+1}{n}} \mathscr{F}(z) du  \,  \land  \frac{\Psi_{\delta}(e^{-k}z^{n})}{\bigvee\limits_{k=\lceil{n\log{}a}\rceil}^{\lfloor{n\log{}b}\rfloor}\Psi_{\delta}(e^{-k}z^{n})}  - \mathscr{F}(z)\right|\\ 
        = & \left|\bigvee\limits_{k=\lceil{n\log{}a}\rceil}^{\lfloor{n\log{}b}\rfloor} \mathscr{F}(z)  \,  \land  \frac{\Psi_{\delta}(e^{-k}z^{n})}{\bigvee\limits_{k=\lceil{n\log{}a}\rceil}^{\lfloor{n\log{}b}\rfloor}\Psi_{\delta}(e^{-k}z^{n})}  - \mathscr{F}(z)\right| \\
        = & \left|\bigvee\limits_{k=\lceil{n\log{}a}\rceil}^{\lfloor{n\log{}b}\rfloor} (\mathscr{F}(z) \land 1) \,  \land  \frac{\Psi_{\delta}(e^{-k}z^{n})}{\bigvee\limits_{k=\lceil{n\log{}a}\rceil}^{\lfloor{n\log{}b}\rfloor}\Psi_{\delta}(e^{-k}z^{n})}  - \mathscr{F}(z)\right|\\ 
        \leq & \left| \mathscr{F}(z)\, \land \bigvee\limits_{k=\lceil{n\log{}a}\rceil}^{\lfloor{n\log{}b}\rfloor}  1 \,  \land  \frac{\Psi_{\delta}(e^{-k}z^{n})}{\bigvee\limits_{k=\lceil{n\log{}a}\rceil}^{\lfloor{n\log{}b}\rfloor}\Psi_{\delta}(e^{-k}z^{n})}  - \mathscr{F}(z)\right|\\ 
        \leq & \left| \mathscr{F}(z)\, \land \bigvee\limits_{k=\lceil{n\log{}a}\rceil}^{\lfloor{n\log{}b}\rfloor}   \frac{\Psi_{\delta}(e^{-k}z^{n})}{\bigvee\limits_{k=\lceil{n\log{}a}\rceil}^{\lfloor{n\log{}b}\rfloor}\Psi_{\delta}(e^{-k}z^{n})}  - \mathscr{F}(z)\right|\\ 
        \leq & \left| \mathscr{F}(z)\, \land 1  - \mathscr{F}(z)\right|\\ 
        = &\, 0.
    \end{align*}
    By combining the inequalities $\,\mathbb{I}_{1} \,\And \, \mathbb{I}_{2}$ mentioned above we obtain the desired result.
\end{proof}

\subsection{Order Of Convergence}\label{subsec3}

In this section, We are using logarithmic Modulus of smoothness to obtain the rate of convergence for the sampling operators $\mathfrak{MK}_{n}^{(m)}$.
\begin{theorem}\label{thm2}
    Let $\mathscr{F}:\mathscr{I} \to [0,1]$,and let $\rho_{n}$ be a null sequence with  $\frac{1}{n\rho_{n}} \to 0 $ as $n \to \infty$ and $v >0$ satisfies the condition $\Delta3$. then for any $\mathscr{F} \in \mathcal{U}_{b}(\mathscr{I})$, we have 
    \[ \|\mathfrak{MK}_{n}^{(m)}(\mathscr{F})- \mathscr{F}\| \leq \Omega_{\mathscr{I}}(\mathscr{F},\rho_{n}) + \left( \Omega_{\mathscr{I}}(\mathscr{F},\rho_{n}) \bigvee \frac{\mathfrak{A}_{v}(\Psi_{\delta})}{\Psi_{\delta}(e)\, n^{v}(\rho_{n})^{v}}\right). \]
\end{theorem}

\begin{proof}
    $\|\mathfrak{MK}_{n}^{(m)}(\mathscr{F})- \mathscr{F} \|$
    \begin{align*}
         \leq &\bigvee\limits_{k=\lceil{n\log{}a}\rceil}^{\lfloor{n\log{}b}\rfloor} n \int_{\frac{k}{n}}^{\frac{k+1}{n}} \left|\left(\mathscr{F}(e^{u}) - \mathscr{F}(z)\right|\right)\,du\,  \land  \frac{\Psi_{\delta}(e^{-k}z^{n})}{\bigvee\limits_{k=\lceil{n\log{}a}\rceil}^{\lfloor{n\log{}b}\rfloor}\Psi_{\delta}(e^{-k}z^{n})}\\
        \leq &\bigvee\limits_{k=\lceil{n\log{}a}\rceil}^{\lfloor{n\log{}b}\rfloor} n \int_{\frac{k}{n}}^{\frac{k+1}{n}} \left|\left(\mathscr{F}(e^{u}) -\mathscr{F}(e^{\frac{k}{n}}) + \mathscr{F}(e^{\frac{k}{n}})- \mathscr{F}(z)\right|\right)\,du\,  \land  \frac{\Psi_{\delta}(e^{-k}z^{n})}{\bigvee\limits_{k=\lceil{n\log{}a}\rceil}^{\lfloor{n\log{}b}\rfloor}\Psi_{\delta}(e^{-k}z^{n})}\\
    \end{align*}
    using the lemma-\ref{lm7} and the triangle inequality, we have:\\[2pt]
    $\|\mathfrak{MK}_{n}^{(m)}(\mathscr{F})- \mathscr{F} \|$
    \begin{align*}
         \leq & \bigvee\limits_{k=\lceil{n\log{}a}\rceil}^{\lfloor{n\log{}b}\rfloor} n \int_{\frac{k}{n}}^{\frac{k+1}{n}} \left|\left(\mathscr{F}(e^{u}) - \mathscr{F}(e^{\frac{k}{n}})\right|\right)\,du\,  \land  \frac{\Psi_{\delta}(e^{-k}z^{n})}{\bigvee\limits_{k=\lceil{n\log{}a}\rceil}^{\lfloor{n\log{}b}\rfloor}\Psi_{\delta}(e^{-k}z^{n})}\\& + \bigvee\limits_{k=\lceil{n\log{}a}\rceil}^{\lfloor{n\log{}b}\rfloor} n \int_{\frac{k}{n}}^{\frac{k+1}{n}} \left|\left(\mathscr{F}(e^{\frac{k}{n}}) - \mathscr{F}(z)\right|\right)\,du\,  \land  \frac{\Psi_{\delta}(e^{-k}z^{n})}{\bigvee\limits_{k=\lceil{n\log{}a}\rceil}^{\lfloor{n\log{}b}\rfloor}\Psi_{\delta}(e^{-k}z^{n})}\\
         \leq &\bigvee\limits_{k=\lceil{n\log{}a}\rceil}^{\lfloor{n\log{}b}\rfloor} n \int_{\frac{k}{n}}^{\frac{k+1}{n}} \left(\left|\mathscr{F}(e^{u}) - \mathscr{F}(e^{\frac{k}{n}})\right|\right)\,du\,  \land  \frac{\Psi_{\delta}(e^{-k}z^{n})}{\bigvee\limits_{k=\lceil{n\log{}a}\rceil}^{\lfloor{n\log{}b}\rfloor}\Psi_{\delta}(e^{-k}z^{n})}\\& + \bigvee\limits_{k=\lceil{n\log{}a}\rceil}^{\lfloor{n\log{}b}\rfloor} \left|\mathscr{F}(e^{\frac{k}{n}}) - \mathscr{F}(z)\right|\,  \land  \frac{\Psi_{\delta}(e^{-k}z^{n})}{\bigvee\limits_{k=\lceil{n\log{}a}\rceil}^{\lfloor{n\log{}b}\rfloor}\Psi_{\delta}(e^{-k}z^{n})}\\
         \leq & \mathbb{R}_{1} +  \mathbb{R}_{2}.
    \end{align*}
    From Definition \ref{def4} we have:\\
    \begin{align*}
        \mathbb{R}_{1} \leq & \bigvee\limits_{k=\lceil{n\log{}a}\rceil}^{\lfloor{n\log{}b}\rfloor}\Omega_{\mathscr{I}}\left( \mathscr{F}, \left|e^{u} -e^{\frac{k}{n}}\right|\right) ~  \land  \frac{\Psi_{\delta}(e^{-k}z^{n})}{\bigvee\limits_{k=\lceil{n\log{}a}\rceil}^{\lfloor{n\log{}b}\rfloor}\Psi_{\delta}(e^{-k}z^{n})} \\
        \leq & \bigvee\limits_{k=\lceil{n\log{}a}\rceil}^{\lfloor{n\log{}b}\rfloor}\Omega_{\mathscr{I}}\left( \mathscr{F}, \left|u -\frac{k}{n}\right|\right)~  \land  \frac{\Psi_{\delta}(e^{-k}z^{n})}{\bigvee\limits_{k=\lceil{n\log{}a}\rceil}^{\lfloor{n\log{}b}\rfloor}\Psi_{\delta}(e^{-k}z^{n})} \\
        \leq & \bigvee\limits_{k=\lceil{n\log{}a}\rceil}^{\lfloor{n\log{}b}\rfloor}\Omega_{\mathscr{I}}\left( \mathscr{F}, \left|u -\frac{k}{n}\right|\right) \\
        \leq & \Omega_{\mathscr{I}}\left( \mathscr{F}, \rho_{n}\right).
    \end{align*}
    Now considering $\mathbb{R}_{2}$ we have:\\
    \begin{align*}
        \mathbb{R}_{2} \leq &\bigvee\limits_{k\in \mathcal{B}_{\gamma, n}} \left| \mathscr{F}(e^{\frac{k}{n}}) - \mathscr{F}(z) \right| \,  \land  \frac{\Psi_{\delta}(e^{-k}z^{n})}{\bigvee\limits_{k=\lceil{n\log{}a}\rceil}^{\lfloor{n\log{}b}\rfloor}\Psi_{\delta}(e^{-k}z^{n})}\\& \bigvee \bigvee\limits_{k\notin \mathcal{B}_{\gamma, n}} \left| \mathscr{F}(e^{\frac{k}{n}}) - \mathscr{F}(z) \right| \,  \land  \frac{\Psi_{\delta}(e^{-k}z^{n})}{\bigvee\limits_{k=\lceil{n\log{}a}\rceil}^{\lfloor{n\log{}b}\rfloor}\Psi_{\delta}(e^{-k}z^{n})} \\
        \leq & max\{R_{2.1},R_{2.2}\}.
    \end{align*}
    From the definition of logarithmic modulus of contunuity~\ref{def4}, we can write $R_{2.1}$ as follows:
    \begin{align*}
        R_{2.1} \leq &\bigvee\limits_{k\in \mathcal{B}_{\gamma, n}} \left| \mathscr{F}(e^{\frac{k}{n}}) - \mathscr{F}(z) \right| \,  \land  \frac{\Psi_{\delta}(e^{-k}z^{n})}{\bigvee\limits_{k=\lceil{n\log{}a}\rceil}^{\lfloor{n\log{}b}\rfloor}\Psi_{\delta}(e^{-k}z^{n})}\\
        \leq & \Omega_{\mathscr{I}}\left( \mathscr{F}, \rho_{n}\right).
    \end{align*}
  Now for $R_{2.2}$, we have:
  \begin{align*}
    \mathbb{R}_{2.2} \leq &\bigvee\limits_{k\notin \mathcal{B}_{\gamma, n}} \left| \mathscr{F}(e^{\frac{k}{n}}) - \mathscr{F}(z) \right| \,  \land  \frac{\Psi_{\delta}(e^{-k}z^{n})}{\bigvee\limits_{k=\lceil{n\log{}a}\rceil}^{\lfloor{n\log{}b}\rfloor}\Psi_{\delta}(e^{-k}z^{n})} \\
    \leq &\bigvee\limits_{k\notin \mathcal{B}_{\gamma, n}}  \frac{\Psi_{\delta}(e^{-k}z^{n})}{\bigvee\limits_{k=\lceil{n\log{}a}\rceil}^{\lfloor{n\log{}b}\rfloor}\Psi_{\delta}(e^{-k}z^{n})} \\
    \leq & \frac{1}{\Psi_{\delta}(e)} \, \bigvee\limits_{\left| n \log{}z -k\right| > n\rho_{n}} \Psi_{\delta}(e^{-k}z^{n}).
\end{align*}
    From definition \ref{def3} we can conclude that $\frac{\left| n \log{}z -k\right|^{v}}{n^{v}\, (\rho_{n})^{v}} > 1 $, where $v$ corresponds to $(\Delta3)$.\\
    Using the above fact, we have 
    \begin{align*}
        \mathbb{R}_{2.2} < & \frac{1}{\Psi_{\delta}(e)\, n^{v}\, (\rho_{n})^{v}} \bigvee\limits_{\left| n \log{}z -k\right| > n\rho_{n}} \Psi_{\delta}(e^{-k}z^{n}) {\left| n \log{}z -k\right|^{v}}\\
        < & \frac{1}{\Psi_{\delta}(e)\, n^{v}\, (\rho_{n})^{v}} \bigvee\limits_{k \in \mathbb{Z}} \Psi_{\delta}(e^{-k}z^{n}) {\left| n \log{}z -k\right|^{v}}\\
        = & \frac{1}{\Psi_{\delta}(e)\, n^{v}\, (\rho_{n})^{v}} \mathfrak{A}_{v}(\Psi_{\delta}).
    \end{align*}
    From all the above inequalities, we obtain the desired result.
\end{proof}

\begin{definition}
   Let $\tau \in (0,1]$. We define the space $\mathfrak{L}_{\log}^\tau$ of logarithmic Hölder continuous functions of order $\tau$ by
\[
\mathfrak{L}_{\log}^\tau := \left\{ \mathscr{F} \in \mathcal{U}_{b}(\mathscr{I}) : \exists \lambda > 0 \ \text{such that} \ |\mathscr{F}(z) - \mathscr{F}(y)| \leq \lambda  |\log z - \log y|^\tau, \quad \forall z,y \in [a,b] \right\}.
\] 
\end{definition} 
In the context of the above theory, if we consider the functions belonging to the space $\mathfrak{L}_{\log}^\tau$, then we obtain the following rate of approximation.

\begin{theorem}
    Let $\mathscr{F} \in \mathfrak{L}_{\log}^\tau$. Then
\[ \| \mathfrak{MK}_n^{(m)}(\mathscr{F}) - \mathscr{F} \|_{\infty} = \mathcal{O}\left( n^{-\frac{\tau}{1+\tau}} \right) \quad \text{as} \quad n \to \infty.\]
\end{theorem}

\begin{proof} 
\text{From} the  Definition \ref{kantdef1} and Lemma \ref{lm6}, we will get :\\

$ \left| \mathfrak{MK}_{n}^{(m)}(\mathscr{F};z)-\mathscr{F}(z) \right|$
\begin{align*}
    = &  \left| \bigvee\limits_{k=\lceil{n\log{a}}\rceil}^{\lfloor{n\log{b}}\rfloor} n 
         \int\limits_{\frac{k}{n}}^{\frac{k+1}{n}} \mathscr{F}(e^{u}) du \quad \land  \frac{\Psi_{\delta}(e^{-k} z^{n})}{\bigvee\limits_{p=\lceil{n\log{a}}\rceil}^{\lfloor{n\log{b}}\rfloor}\Psi_{\delta}(e^{-p} z^{n})} - \bigvee\limits_{k=\lceil{n\log{a}}\rceil}^{\lfloor{n\log{b}}\rfloor} n \int\limits_{\frac{k}{n}}^{\frac{k+1}{n}} \mathscr{F}(z) du \quad \land  \frac{\Psi_{\delta}(e^{-k} z^{n})}{\bigvee\limits_{p=\lceil{n\log{a}}\rceil}^{\lfloor{n\log{b}}\rfloor}\Psi_{\delta}(e^{-p} z^{n})} \right|\\
    \leq & \bigvee\limits_{k=\lceil{n\log{a}}\rceil}^{\lfloor{n\log{b}}\rfloor}\left|n 
         \int\limits_{\frac{k}{n}}^{\frac{k+1}{n}} \mathscr{F}(e^{u}) du  - n \int\limits_{\frac{k}{n}}^{\frac{k+1}{n}} \mathscr{F}(z) du \right| \quad \land  \frac{\Psi_{\delta}(e^{-k} z^{n})}{\bigvee\limits_{p=\lceil{n\log{a}}\rceil}^{\lfloor{n\log{b}}\rfloor}\Psi_{\delta}(e^{-p} z^{n})} \\
    \leq & \bigvee\limits_{k=\lceil{n\log{a}}\rceil}^{\lfloor{n\log{b}}\rfloor}n 
         \int\limits_{\frac{k}{n}}^{\frac{k+1}{n}} \left|\mathscr{F}(e^{u})- \mathscr{F}(z)\right| du  \quad \land  \frac{\Psi_{\delta}(e^{-k} z^{n})}{\bigvee\limits_{p=\lceil{n\log{a}}\rceil}^{\lfloor{n\log{b}}\rfloor}\Psi_{\delta}(e^{-p} z^{n})}\\
\end{align*}

Using the inequality:
\[|\mathscr{F}(e^{u}) - \mathscr{F}(z)| \leq |\mathscr{F}(e^{u}) - \mathscr{F}(e^{\frac{k}{n}})| + |\mathscr{F}(e^{\frac{k}{n}}) - \mathscr{F}(z)|,\]
and the elementary inequality:
\[(a + b) \land  c \leq (a \land c) + (b \land c), \quad \text{for all } a, b, c \geq 0,\]
we further obtain:\\[1pt]
$ \left| \mathfrak{MK}_{n}^{(m)}(\mathscr{F};z)-\mathscr{F}(z) \right|$
\begin{align*}
    \leq & \bigvee\limits_{k=\lceil{n\log{a}}\rceil}^{\lfloor{n\log{b}}\rfloor}n 
         \int\limits_{\frac{k}{n}}^{\frac{k+1}{n}} \left|\mathscr{F}(e^{u})- \mathscr{F}(e^{\frac{k}{n}})\right| du  \quad \land  \frac{\Psi_{\delta}(e^{-k} z^{n})}{\bigvee\limits_{p=\lceil{n\log{a}}\rceil}^{\lfloor{n\log{b}}\rfloor}\Psi_{\delta}(e^{-p} z^{n})}\\ & +  \bigvee\limits_{k=\lceil{n\log{a}}\rceil}^{\lfloor{n\log{b}}\rfloor}n 
         \int\limits_{\frac{k}{n}}^{\frac{k+1}{n}} \left| \mathscr{F}(e^{\frac{k}{n}})-\mathscr{F}(z)\right| du  \quad \land  \frac{\Psi_{\delta}(e^{-k} z^{n})}{\bigvee\limits_{p=\lceil{n\log{a}}\rceil}^{\lfloor{n\log{b}}\rfloor}\Psi_{\delta}(e^{-p} z^{n})}.\\
\end{align*}
Since $\mathscr{F} \in \mathfrak{L}_{\log}^\tau$ , we will get:\\[2pt]
$ \left| \mathfrak{MK}_{n}^{(m)}(\mathscr{F};z)-\mathscr{F}(z) \right|$
\begin{align*}
    \leq & \bigvee\limits_{k=\lceil{n\log{a}}\rceil}^{\lfloor{n\log{b}}\rfloor}n 
         \int\limits_{\frac{k}{n}}^{\frac{k+1}{n}} \lambda \left|(u-\frac{k}{n})\right|^{\tau} du  \quad \land  \frac{\Psi_{\delta}(e^{-k} z^{n})}{\bigvee\limits_{p=\lceil{n\log{a}}\rceil}^{\lfloor{n\log{b}}\rfloor}\Psi_{\delta}(e^{-p} z^{n})}\\ & +  \bigvee\limits_{k=\lceil{n\log{a}}\rceil}^{\lfloor{n\log{b}}\rfloor}n 
         \int\limits_{\frac{k}{n}}^{\frac{k+1}{n}} \lambda \left| \frac{k}{n}-\log{z}\right|^{\tau} du  \quad \land  \frac{\Psi_{\delta}(e^{-k} z^{n})}{\bigvee\limits_{p=\lceil{n\log{a}}\rceil}^{\lfloor{n\log{b}}\rfloor}\Psi_{\delta}(e^{-p} z^{n})}\\
    = &  l_{1} + l_{2}.
\end{align*}
For some $\lambda >0,$  the first term simplifies as: \[l_{1} = \frac{\lambda}{(\tau +1)n^{\tau}}\]

By choosing $\rho_{n} = \frac{1}{n^{\frac{1}{1+\tau}}}$ the second term can be decomposed  as follows:
\begin{align*}
 l_{2} = & \bigvee\limits_{k\in\mathcal{B}_{\rho_{n}, n}}\lambda \,n 
           \int\limits_{\frac{k}{n}}^{\frac{k+1}{n}}  \left| \frac{k}{n}-\log{z}\right|^{\tau} du  \quad \land  \frac{\Psi_{\delta}(e^{-k} z^{n})}{\bigvee\limits_{p=\lceil{n\log{a}}\rceil}^{\lfloor{n\log{b}}\rfloor}\Psi_{\delta}(e^{-p} z^{n})} \\ & \bigvee \bigvee\limits_{k\notin\mathcal{B}_{\rho_{n}, n}}\lambda \,n \int\limits_{\frac{k}{n}}^{\frac{k+1}{n}}  \left| \frac{k}{n}-\log{z}\right|^{\tau} du  \quad \land  \frac{\Psi_{\delta}(e^{-k} z^{n})}{\bigvee\limits_{p=\lceil{n\log{a}}\rceil}^{\lfloor{n\log{b}}\rfloor}\Psi_{\delta}(e^{-p} z^{n})}  \\
       = & \quad max\{l_{2_{1}}, l_{2_{2}} \}.
\end{align*}
Using definition of $\mathcal{B}_{\rho_{n}, n}$, the term $l_{2_{1}}$ can be written as follows 
\begin{align*}
   l_{2_{1}} = &\bigvee\limits_{k\in\mathcal{B}_{\rho_{n}, n}}\lambda \,n \int\limits_{\frac{k}{n}}^{\frac{k+1}{n}}  \left| \frac{k}{n}-\log{z}\right|^{\tau} du  \quad \land  \frac{\Psi_{\delta}(e^{-k} z^{n})}{\bigvee\limits_{p=\lceil{n\log{a}}\rceil}^{\lfloor{n\log{b}}\rfloor}\Psi_{\delta}(e^{-p} z^{n})} \\
   \leq & \bigvee\limits_{k\in\mathcal{B}_{\rho_{n}, n}}\lambda \,n \int\limits_{\frac{k}{n}}^{\frac{k+1}{n}} \rho_{n}^{\tau} \quad du  \quad \land  \frac{\Psi_{\delta}(e^{-k} z^{n})}{\bigvee\limits_{p=\lceil{n\log{a}}\rceil}^{\lfloor{n\log{b}}\rfloor}\Psi_{\delta}(e^{-p} z^{n})}
   \leq  \frac{\lambda}{n^{\frac{\tau}{1+ \tau}}}.
\end{align*}
Now the term $l_{2_{2}}$ can be written as 
\begin{align*}
    l_{2_{2}} = & \bigvee\limits_{k\notin\mathcal{B}_{\rho_{n}, n}}\lambda \,n \int\limits_{\frac{k}{n}}^{\frac{k+1}{n}}  \left| \frac{k}{n}-\log{z}\right|^{\tau} du  \quad \land  \frac{\Psi_{\delta}(e^{-k} z^{n})}{\bigvee\limits_{p=\lceil{n\log{a}}\rceil}^{\lfloor{n\log{b}}\rfloor}\Psi_{\delta}(e^{-p} z^{n})}\\
    \leq & \bigvee\limits_{k\notin\mathcal{B}_{\rho_{n}, n}}\,\frac{\Psi_{\delta}(e^{-k} z^{n})}{\bigvee\limits_{p=\lceil{n\log{a}}\rceil}^{\lfloor{n\log{b}}\rfloor}\Psi_{\delta}(e^{-p} z^{n})} 
    \leq  \bigvee\limits_{k\notin\mathcal{B}_{\rho_{n}, n}}\,\frac{\Psi_{\delta}(e^{-k} z^{n})}{\Psi_{\delta}(e)} \frac{\left|n\log{z}- k\right|}{n\rho_{n}}
    \leq  \frac{\mathfrak{A}_{1}(\Psi_{\delta})}{\Psi_{\delta}(e)} \frac{1}{n^{\frac{\tau}{1+\tau}}}.
\end{align*}
Hence, combining all estimates , we conclude:
\[ \| \mathfrak{MK}_n^{(m)}(\mathscr{F}) - \mathscr{F} \|_{\infty} = \mathcal{O}\left( n^{-\frac{\tau}{1+\tau}} \right) \quad \text{as} \quad n \to \infty.\]
\end{proof}

Now our analysis considers the convergence of Kantorovich type exponential sampling NN operators in the  Mellin Orlicz space.
\subsection{Convergence Results in Mellin Orlicz Space}
\begin{theorem}\label{thm1orlicz}
    Let $\mathscr{F} \in \mathcal{U}_{b}(\mathscr{I})$ . Then, for each $\lambda > 0$, we have:
\[\lim_{n \to \infty} \mathrm{I}_{\eta}\left[\lambda \left( (\mathfrak{MK}_n^{(m)} \mathscr{F}) - \mathscr{F} \right) \right] = 0.\]
\end{theorem}
\begin{proof}
Let \(\varepsilon > 0\) be fixed. Then, for any \(\lambda > 0\), by applying  the convexity of \(\eta\), and referring to Theorem \ref{Kantthm1}, we obtain that:
\begin{align*}
\mathrm{I}_{\eta}\left[\lambda \left( (\mathfrak{MK}_n^{(m)} \mathscr{F}) - \mathscr{F} \right) \right] = & \int\limits_{\mathscr{I}} \eta \left( \left| \lambda \left( (\mathfrak{MK}_n^{(m)} \mathscr{F})(z) - \mathscr{F}(z) \right)\right|\right) \frac{dz}{z}\\
\leq & \int\limits_{\mathscr{I}} \eta \left( \lambda\left\| \left( (\mathfrak{MK}_n^{(m)} \mathscr{F}) - \mathscr{F} \right)\right\|_{\infty}\right) \frac{dz}{z} \\
\leq & \int\limits_{\mathscr{I}} \eta(\lambda \, \varepsilon)\, \frac{dz}{z} \\
= & \, \eta (\lambda \,\varepsilon) (\log{b} - \log{a}) \\
\leq &\,  \varepsilon \,\eta(\lambda)(\log{b} - \log{a}).
\end{align*}
Hence, the result follows by the arbitrariness of \(\varepsilon > 0\).
\end{proof}

\begin{lemma}\label{}(see\cite{raoren1991})
   The space $ \mathcal{U}_{b}(\mathscr{I})$  is modularly dense in \( L_{\mu}^\eta(\mathscr{I}) \).  
\end{lemma}

\begin{theorem}\label{orlthm1}
    For every $\mathscr{F},\mathscr{G}\in L_{\mu}^\eta(\mathscr{I})\,\text{and }\lambda >0 $,  we have 
    \[\mathrm{I}_{\eta}\left[\lambda \left( (\mathfrak{MK}_{n}^{(m)}\mathscr{F}) - (\mathfrak{MK}_{n}^{(m)}\mathscr{G}) \right) \right] \leq 2 (I_{\eta}(\lambda(\mathscr{F} - \mathscr{G})))^{\frac{1}{1+\beta}} + \frac{\eta(\lambda)}{\Psi_{\delta}(e)} \varepsilon \left(\lceil{n\log b\rceil}- \lfloor{n\log a\rfloor} \right). \]
\end{theorem}
\begin{proof}
Using the definition of modular functional given in~\ref{MFdef} , we have:
\begin{align*}
    \mathrm{I}_{\eta}\left[\lambda \left( (\mathfrak{MK}_{n}^{(m)}\mathscr{F}) - (\mathfrak{MK}_{n}^{(m)}\mathscr{G}) \right) \right] = & \int\limits_{a}^{b} \eta\left(\left|  \lambda \left( \mathfrak{MK}_{n}^{(m)}\mathscr{F}) - (\mathfrak{MK}_{n}^{(m)}\mathscr{G} \right)\right| \right)\frac{dz}{z}\\
     \leq & \int\limits_{a}^{b} \eta\left(\lambda\left|   \left( \mathfrak{MK}_{n}^{(m)}\mathscr{F}) - (\mathfrak{MK}_{n}^{(m)}\mathscr{G} \right)\right| \right)\frac{dz}{z}.
\end{align*}
\vspace{1em}
Then by (c) of lemma-\ref{lm1kant}, we get\\
$\mathrm{I}_{\eta}\left[\lambda \left( (\mathfrak{MK}_{n}^{(m)}\mathscr{F}) - (\mathfrak{MK}_{n}^{(m)}\mathscr{G}) \right) \right]$
\begin{align*}
      \leq & \int\limits_{a}^{b} \eta\left(\lambda   \left( \mathfrak{MK}_{n}^{(m)}\left|\mathscr{F} - \mathscr{G} \right| \right)\right)\frac{dz}{z}\\
     = & \int\limits_{a}^{b} \eta\left( \lambda \bigvee\limits_{k \in \mathscr{J}_{n}} \left[ n\int\limits_{\frac{k}{n}}^{\frac{k+1}{n}} \left|\mathscr{F}(e^{u})-\mathscr{G}(e^{u})\right|  \,du \ \land \ \frac{\Psi_\delta(n\log z -k)}{\bigvee\limits_{j\in \mathscr{J}_{n}} \Psi_\delta(n\log z -j)} \right] \right) \frac{dz}{z}\\
      = & \int\limits_{a}^{b} \eta\left( \bigvee\limits_{k \in \mathscr{J}_{n}} \left[ \lambda n\int\limits_{\frac{k}{n}}^{\frac{k+1}{n}} \left|\mathscr{F}(e^{u})-\mathscr{G}(e^{u})\right|  \,du \ \land \ \frac{\lambda \, \Psi_\delta(n\log z -k)}{\bigvee\limits_{j\in \mathscr{J}_{n}} \Psi_\delta(n\log z -j)} \right] \right) \frac{dz}{z}.\\
\end{align*}
    Since $\eta$ is a non-decreasing function and $\mathscr{J}_n = \lceil{n\log a\rceil}\dots \lfloor{n\log b\rfloor}$ is a finite set, we have:
$\varphi\left(\bigvee\limits_{k \in \mathscr{J}_n} a_k \right) = \bigvee\limits_{k \in \mathscr{J}_n} \varphi(a_k).$ and $\varphi\left(\bigwedge\limits_{k \in \mathscr{J}_n} a_k \right) = \bigwedge\limits_{k \in \mathscr{J}_n} \varphi(a_k).$\\

Therefore, we get\\
$\mathrm{I}_{\eta}\left[\lambda \left( (\mathfrak{MK}_{n}^{(m)}\mathscr{F}) - (\mathfrak{MK}_{n}^{(m)}\mathscr{G}) \right) \right]$
\begin{align*}
    = & \int\limits_{a}^{b} \left( \bigvee\limits_{k \in \mathscr{J}_{n}} \eta\left( \lambda n\int\limits_{\frac{k}{n}}^{\frac{k+1}{n}} \left|\mathscr{F}(e^{u})-\mathscr{G}(e^{u})\right|  \,du \ \land \ \frac{\lambda \, \Psi_\delta(n\log z -k)}{\bigvee\limits_{j\in \mathscr{J}_{n}} \Psi_\delta(n\log z -j)} \right) \right) \frac{dz}{z}\\
    = & \int\limits_{a}^{b} \left( \bigvee\limits_{k \in \mathscr{J}_{n}} \eta\left( \lambda n\int_{\frac{k}{n}}^{\frac{k+1}{n}} \left|\mathscr{F}(e^{u})-\mathscr{G}(e^{u})\right|  \,du \right) \land \, \eta \left( \frac{\lambda \, \Psi_\delta(n\log z -k)}{\bigvee\limits_{j\in \mathscr{J}_{n}} \Psi_\delta(n\log z -j)} \right) \right)\frac{dz}{z}\\
    = & \int\limits_{a}^{b} \left( \bigvee\limits_{k \in \mathscr{J}_{n}} \eta\left(  n\int_{\frac{k}{n}}^{\frac{k+1}{n}} \lambda\,\left|\mathscr{F}(e^{u})-\mathscr{G}(e^{u})\right|  \,du \right) \land \, \eta \left( \frac{\lambda \, \Psi_\delta(n\log z -k)}{\bigvee\limits_{j\in \mathscr{J}_{n}} \Psi_\delta(n\log z -j)} \right) \right)\frac{dz}{z}.\\
\end{align*}
By applying the \textit{Jensen's Inequality}, we obtain\\
\vspace{1em}
$\mathrm{I}_{\eta}\left[\lambda \left( (\mathfrak{MK}_{n}^{(m)}\mathscr{F}) - (\mathfrak{MK}_{n}^{(m)}\mathscr{G}) \right) \right]$
\begin{align*}
    \leq & \int\limits_{a}^{b}  \bigvee\limits_{k \in \mathscr{J}_{n}}  \left\{ n\int_{\frac{k}{n}}^{\frac{k+1}{n}} \eta\left( \lambda\,\left|\mathscr{F}(e^{u})-\mathscr{G}(e^{u})\right|\right)  \,du \, \land \, \eta \left( \frac{\lambda \, \Psi_\delta(n\log z -k)}{\bigvee\limits_{j\in \mathscr{J}_{n}} \Psi_\delta(n\log z -j)} \right) \right\}\frac{dz}{z}\\
    = & \int\limits_{a}^{b}  \bigvee\limits_{k \in \mathscr{J}_{n}}  n\left\{\int_{\frac{k}{n}}^{\frac{k+1}{n}} \eta\left( \lambda\,\left|\mathscr{F}(e^{u})-\mathscr{G}(e^{u})\right|\right)  \,du \, \land \, \eta \left( \frac{\lambda \, \Psi_\delta(n\log z -k)}{\bigvee\limits_{j\in \mathscr{J}_{n}} \Psi_\delta(n\log z -j)} \right) \right\}\frac{dz}{z}\\
    \leq & \int\limits_{\mathbb{R}_{+}} \sum\limits_{k \in \mathscr{J}_{n}}  n\left\{\int_{\frac{k}{n}}^{\frac{k+1}{n}} \eta\left( \lambda\,\left|\mathscr{F}(e^{u})-\mathscr{G}(e^{u})\right|\right)  \,du \, \land \, \eta \left( \frac{\lambda \, \Psi_\delta(n\log z -k)}{\bigvee\limits_{j\in \mathscr{J}_{n}} \Psi_\delta(n\log z -j)} \right) \right\}\frac{dz}{z}.\\
\end{align*}
Applying \texttt{Fubini-Tonelli's Theorem}, we establish that\\
$\mathrm{I}_{\eta}\left[\lambda \left( (\mathfrak{MK}_{n}^{(m)}\mathscr{F}) - (\mathfrak{MK}_{n}^{(m)}\mathscr{G}) \right) \right]$
\begin{align*}
    \leq &  \sum\limits_{k \in \mathscr{J}_{n}}\int\limits_{\mathbb{R}_{+}}  n\left\{\int_{\frac{k}{n}}^{\frac{k+1}{n}} \eta\left( \lambda\,\left|\mathscr{F}(e^{u})-\mathscr{G}(e^{u})\right|\right)  \,du \, \land \, \eta \left( \frac{\lambda \, \Psi_\delta(n\log z -k)}{\bigvee\limits_{j\in \mathscr{J}_{n}} \Psi_\delta(n\log z -j)} \right) \right\}\frac{dz}{z}\\
    \leq &  \sum\limits_{k \in \mathscr{J}_{n}}\int\limits_{\mathbb{R}_{+}}  n\left\{\int_{\frac{k}{n}}^{\frac{k+1}{n}} \eta(\left( \lambda\,\left|\mathscr{F}(e^{u})-\mathscr{G}(e^{u})\right|\right)  \,du \, \land \, \eta(\lambda) \left( \frac{ \Psi_\delta(n\log z -k)}{\bigvee\limits_{j\in \mathscr{J}_{n}} \Psi_\delta(n\log z -j)} \right) \right\}\frac{dz}{z}\\
    =:& \sum\limits_{k \in \mathscr{J}_{n}} \mathcal{I}_{k}.
\end{align*}
Now for fixed $m \in \mathscr{J}_{n}$, the integral $\mathcal{I}_{m}$ is given by
\[\mathcal{I}_{m}=\int\limits_{\mathbb{R}_{+}}  n\left\{\int_{\frac{m}{n}}^{\frac{m+1}{n}} \eta(\left( \lambda\,\left|\mathscr{F}(e^{u})-\mathscr{G}(e^{u})\right|\right)  \,du \, \land \, \eta(\lambda) \left( \frac{ \Psi_\delta(n\log z -m)}{\bigvee\limits_{j\in \mathscr{J}_{n}} \Psi_\delta(n\log z -j)} \right) \right\}\frac{dz}{z}.\]
By setting $n\log z -m = v_{m}$ and applying lemma-\ref{lm2}, we obtain \\
\[\mathcal{I}_{m}=\int\limits_{\mathbb{R}_{+}}  \left\{\int_{\frac{m}{n}}^{\frac{m+1}{n}} \eta(\left( \lambda\,\left|\mathscr{F}(e^{u})-\mathscr{G}(e^{u})\right|\right)  \,du \, \land \, \eta(\lambda)  \frac{ \Psi_\delta(v_{m})}{\Psi_{\delta}(e)}  \right\}dv_{m}\, .\]
Since $\Psi_\delta \in L^1(\mathbb{R}_+)$, then for each given $\varepsilon >0$, there exists $n_{m} >0$ such that 
\begin{eqnarray}\label{orleq1}
    \int\limits_{|v_{m}| > n_{m}} \Psi_{\delta}(v_{m}) \, dv_{m} < \varepsilon.
\end{eqnarray}
For some fixed $\beta > 0 $, suppose that the modular functional $I_{\eta}(\lambda(\mathscr{F} - \mathscr{G}))$ to be of such a way that
\[r_{m} := (I_{\eta}(\lambda(\mathscr{F} - \mathscr{G})))^{\frac{-\beta}{1+\beta}} > n_{m}.\]
By partitioning the $\mathcal{I}_{m}$ , we obtain:
 \begin{align*}
    \mathcal{I}_{m} = &   \int\limits_{|v_{m}| \leq r_{m}}  \left\{\int_{\frac{m}{n}}^{\frac{m+1}{n}} \eta(\left( \lambda\,\left|\mathscr{F}(e^{u})-\mathscr{G}(e^{u})\right|\right)  \,du \, \land \, \eta(\lambda)  \frac{ \Psi_\delta(v_{m})}{\Psi_{\delta}(e)}  \right\}dv_{m} \\ & + \int\limits_{|v_{m}| > r_{m}}  \left\{\int_{\frac{m}{n}}^{\frac{m+1}{n}} \eta(\left( \lambda\,\left|\mathscr{F}(e^{u})-\mathscr{G}(e^{u})\right|\right)  \,du \, \land \, \eta(\lambda)  \frac{ \Psi_\delta(v_{m})}{\Psi_{\delta}(e)}  \right\}dv_{m}\\ 
    \leq & \int\limits_{|v_{m}| \leq r_{m}}  \left\{\int_{\frac{m}{n}}^{\frac{m+1}{n}} \eta(\left( \lambda\,\left|\mathscr{F}(e^{u})-\mathscr{G}(e^{u})\right|\right)  \,du\,  \right\}dv_{m} + \int\limits_{|v_{m}| > r_{m}}  \left\{ \eta(\lambda)  \frac{ \Psi_\delta(v_{m})}{\Psi_{\delta}(e)}  \right\}dv_{m}.\\
 \end{align*}
 From \ref{orleq1} and by setting $e^{u} =t $, we have:
\begin{align*}
    \mathcal{I}_{m} \leq & \int\limits_{|v_{m}| \leq r_{m}}  \left\{\int\limits_{e^{\frac{m}{n}}}^{e^{\frac{m+1}{n}}} \eta(\left( \lambda\,\left|\mathscr{F}(t)-\mathscr{G}(t)\right|\right)  \,\frac{dt}{t}\,  \right\}dv_{m} + \int\limits_{|v_{m}| > r_{m}}  \left\{ \eta(\lambda)  \frac{ \Psi_\delta(v_{m})}{\Psi_{\delta}(e)}  \right\}dv_{m}\\
    \leq & 2r_{m} \left\{\int\limits_{e^{\frac{m}{n}}}^{e^{\frac{m+1}{n}}} \eta(\left( \lambda\,\left|\mathscr{F}(t)-\mathscr{G}(t)\right|\right)  \,\frac{dt}{t} \right\} + \frac{\eta(\lambda)}{\Psi_{\delta}(e)} \varepsilon.
 \end{align*}
 Since this is true for each $m\in \mathscr{I}_{n}$, we have:\\[1pt]
 $\mathrm{I}_{\eta}\left[\lambda \left( (\mathfrak{MK}_{n}^{(m)}\mathscr{F}) - (\mathfrak{MK}_{n}^{(m)}\mathscr{G}) \right) \right] $
 \begin{align*}
    \leq & \sum\limits_{k\in \mathscr{I}_{n}}2r_{m} \left\{\int\limits_{e^{\frac{k}{n}}}^{e^{\frac{k+1}{n}}} \eta(\left( \lambda\,\left|\mathscr{F}(t)-\mathscr{G}(t)\right|\right)  \,\frac{dt}{t} \right\} + \frac{\eta(\lambda)}{\Psi_{\delta}(e)} \varepsilon\\
    = & 2r_{m}\sum\limits_{k\in \mathscr{I}_{n}} \left\{\int\limits_{e^{\frac{k}{n}}}^{e^{\frac{k+1}{n}}} \eta(\left( \lambda\,\left|\mathscr{F}(t)-\mathscr{G}(t)\right|\right)  \,\frac{dt}{t} \right\} + \sum\limits_{k\in \mathscr{I}_{n}}\frac{\eta(\lambda)}{\Psi_{\delta}(e)} \varepsilon\\
    \leq & 2 (I_{\eta}(\lambda(\mathscr{F} - \mathscr{G})))^{\frac{-\beta}{1+\beta}} (I_{\eta}(\lambda(\mathscr{F} - \mathscr{G}))) + \frac{\eta(\lambda)}{\Psi_{\delta}(e)} \varepsilon \left(\lceil{n\log b\rceil}- \lfloor{n\log a\rfloor} \right)\\
    = & 2 (I_{\eta}(\lambda(\mathscr{F} - \mathscr{G})))^{\frac{1}{1+\beta}} + \frac{\eta(\lambda)}{\Psi_{\delta}(e)} \varepsilon \left(\lceil{n\log b\rceil}- \lfloor{n\log a\rfloor} \right).
 \end{align*}
\end{proof}

\begin{theorem}\label{}
   For every $ \mathscr{F} \in L_{\mu}^\eta(\mathscr{I})\,\,\text{there exists} \,\,\, \lambda > 0 \,\, $ such that
\[\lim_{n \to \infty} \mathrm{I}_{\eta}\left[\lambda \left( (\mathfrak{MK}_{n}^{(m)}\mathscr{F}) - \mathscr{F} \right) \right] = 0.\]
 
\end{theorem}

\begin{proof}
    Let $\varepsilon > 0 $ be given and choose $\lambda_{*} > 0$. Then by Theorem-\ref{thm1orlicz}, \,\,$\exists \, n_0 \in \mathbb{N}$ such that
\begin{eqnarray}\label{orleq2}
    I_\eta\left[ \lambda_{*} \left( \mathfrak{MK}_n^{(m)} \mathscr{G} - \mathscr{G} \right) \right] < \varepsilon, \quad  \forall n \geq n_0.
\end{eqnarray} 
    Moreover, since the space \( \, \mathcal{U}_{b}(\mathscr{I}) \)  is modularly dense in \( L_{\mu}^\eta(\mathscr{I}) \), for each given $\varepsilon > 0 $ and $ \lambda_{*} > 0\,\, \exists $ a function \( \mathscr{G} \in \mathcal{U}_{b}([a,b]) \) such that
\begin{eqnarray}\label{orleq3}
    \mathrm{I}_{\eta}\left[\lambda_{*} (\mathscr{F} - \mathscr{G})\right] < \varepsilon^{\beta+1}, \, for \,\beta >0.
\end{eqnarray}
\begin{align*}
    \mathrm{I}_{\eta}\left[\lambda \left( (\mathfrak{MK}_{n}^{(m)}\mathscr{F}) - \mathscr{F} \right) \right] = &  \mathrm{I}_{\eta}\left[\lambda \left( (\mathfrak{MK}_{n}^{(m)}\mathscr{F}) - (\mathfrak{MK}_{n}^{(m)}\mathscr{G}) +(\mathfrak{MK}_{n}^{(m)}\mathscr{G})- \mathscr{G} + \mathscr{G}- \mathscr{F}\right) \right]. \\
\end{align*}
since $\eta$ is convex then for each $n > n_{0}$, we get
\begin{align*}
    \mathrm{I}_{\eta}\left[\lambda \left( (\mathfrak{MK}_{n}^{(m)}\mathscr{F}) - \mathscr{F} \right) \right] \leq & \frac{1}{3} \mathrm{I}_{\eta}\left[3\lambda \left( (\mathfrak{MK}_{n}^{(m)}\mathscr{F}) -(\mathfrak{MK}_{n}^{(m)}\mathscr{G})\right)\right] + \frac{1}{3} \mathrm{I}_{\eta}\left[3\lambda(\mathfrak{MK}_{n}^{(m)}\mathscr{G})- \mathscr{G}\right] \\&+ \frac{1}{3}\mathrm{I}_{\eta}\left[3\lambda\left(\mathscr{G}- \mathscr{F}\right)\right] . \\
\end{align*}
By considering   $3\lambda < \lambda_{*}$ for $\lambda > 0 $, we get
\begin{align*}
    \mathrm{I}_{\eta}\left[\lambda \left( (\mathfrak{MK}_{n}^{(m)}\mathscr{F}) - \mathscr{F} \right) \right] \leq & \frac{1}{3} \mathrm{I}_{\eta}\left[\lambda_{*}\left( (\mathfrak{MK}_{n}^{(m)}\mathscr{F}) -(\mathfrak{MK}_{n}^{(m)}\mathscr{G})\right)\right] + \frac{1}{3} \mathrm{I}_{\eta}\left[\lambda_{*}(\mathfrak{MK}_{n}^{(m)}\mathscr{G})- \mathscr{G}\right] \\&+ \frac{1}{3}\mathrm{I}_{\eta}\left[\lambda_{*}\left(\mathscr{G}- \mathscr{F}\right)\right].  \\
\end{align*}
Now by using Theorem-\ref{orlthm1} and from \ref{orleq2},\ref{orleq3}, we obtain :
\begin{align*}
    \mathrm{I}_{\eta}\left[\lambda \left( (\mathfrak{MK}_{n}^{(m)}\mathscr{F}) - \mathscr{F} \right) \right] \leq & \frac{1}{3} \left[2 (\varepsilon^{\beta +1})^{\frac{1}{1+\beta}} + \frac{\eta(\lambda_{*})}{\Psi_{\delta}(e)} \varepsilon \left(\lceil{n\log b\rceil}- \lfloor{n\log a\rfloor} \right)\right] +\frac{\varepsilon}{3} + \frac{\varepsilon^{\beta +1}}{3} \\
    \leq & \frac{1}{3} \left[2\varepsilon + \frac{\eta(\lambda_{*})}{\Psi_{\delta}(e)} \varepsilon \left(\lceil{n\log b\rceil}- \lfloor{n\log a\rfloor} \right)\right] +\frac{\varepsilon}{3} + \frac{\varepsilon^{\beta +1}}{3} .\\
\end{align*}
 The proof is complete due to arbitrariness of  $\varepsilon$.
\end{proof}

\section{Examples of Activation Functions }\label{sec5}

In this section, we have presented several examples of sigmoidal functions along with their corresponding activation (kernel) functions, as summarized in Table~\ref{tab:sigmoid-kernels}. Among these, the \texttt{logistic} and \texttt{hyperbolic tangent-based} sigmoidal functions are smooth and satisfy all three structural conditions: $ \Delta1 $ (odd-centered symmetry), $ \Delta2 $ (concavity on \( \mathbb{R}_+ \)), and \( \Delta3 \) (polynomial-type decay as \( z \to -\infty \)). On the other hand, the \texttt{ramp sigmoidal function} \( \sigma_{R} \) is a non-smooth example that satisfies conditions \( \Delta1 \) and \( \Delta3 \), but not \( \Delta_2 \), due to its piecewise linear structure. However, in view of Remark~\ref{rem2}, which characterizes the kernel \( \varphi_\sigma \) as a centered bell-shaped function, \( \sigma_{R} \) can still be considered a valid activation function in the framework of exponential sampling theory. Additionally, the \texttt{three-level sigmoidal function} \( \sigma_{\text{three}} \) serves as another example of a valid, though non-smooth, sigmoidal function. It is important to note that the kernels derived from \( \sigma_l \)\,and  \,\( \sigma_h \) are not compactly supported, while those from \( \sigma_R \) and \( \sigma_{\text{three}} \) are compactly supported, making them advantageous in local approximation scenarios.\\

\begin{table}[h!]
\centering
\caption{Examples of kernels associated with sigmoidal function}
\label{tab:sigmoid-kernels}
\begin{tabular}{@{}ll@{}}
\toprule
\textbf{Sigmoidal Function \( \sigma(z) \)} & \textbf{Kernel \( \varphi_\sigma(z) = \frac{1}{2}[\sigma(\log z+1) - \sigma(\log z-1)] \)} \\
\midrule

\textbf{1. Logistic} 
\vspace{1mm} \\ 
\( \displaystyle \sigma_l(z) = \frac{1}{1 + e^{-z}} \) & 
\( \displaystyle \varphi_{\sigma_l(z)} =  \frac{z(e^2 - 1)}{2(z + e)(ez + 1)}, \quad z \in \mathbb{R}^+ \)

\vspace{2mm} \\

\textbf{2. Hyperbolic Tangent} 
\vspace{1mm} \\
\( \displaystyle \sigma_h(z) = \frac{1}{2}(\tanh z + 1) \) & 
\( \displaystyle \varphi_{\sigma_h}(z) = \frac{1}{2} \cdot \frac{z^2(e^4 - 1)}{z^2(1 + e^4 + e^2z^2) + e^2}, \quad z \in \mathbb{R}^+ \)

\vspace{2mm} \\

\textbf{3. Ramp Function} 
\vspace{1mm} \\
\( \displaystyle
\sigma_R(z) = 
\begin{cases}
0, & z < -\frac{1}{2}, \\
z + \frac{1}{2}, & -\frac{1}{2} \le z \le \frac{1}{2}, \\
1, & z > \frac{1}{2}
\end{cases}
\) &
\( \displaystyle
\varphi_{\sigma_R}(z) =
\begin{cases}
0, & z < e^{-3/2} \\[6pt]
\frac{1}{2} \left( \log z + \frac{3}{2} \right), & e^{-3/2} \le z < e^{-1/2} \\[6pt]
\frac{1}{2}, & e^{-1/2} \le z \leq e^{1/2} \\[6pt]
\frac{1}{2} \left( -\log z + \frac{3}{2} \right), & e^{1/2} < z \le e^{3/2} \\[6pt]
0, & z > e^{3/2}
\end{cases}
\)

\vspace{2mm} \\

\textbf{4. Three-Level Sigmoid}
\vspace{1mm} \\
\( \displaystyle
\sigma_{\text{three}}(z) =
\begin{cases}
0, & z < -\frac{1}{2}, \\
\frac{1}{2}, & -\frac{1}{2} \le z \le \frac{1}{2}, \\
1, & z > \frac{1}{2}
\end{cases}
\) &
\( \displaystyle
\varphi_{\sigma_{\text{three}}}(z) =
\begin{cases}
\dfrac{1}{4}, & e^{-3/2} \le z < e^{-1/2}, \\[6pt]
\dfrac{1}{2}, & e^{-1/2} \le z \le  e^{1/2}, \\[6pt]
\dfrac{1}{4}, & e^{1/2} < z \le e^{3/2}, \\[6pt]
0, & \text{otherwise}
\end{cases}
\)

\vspace{2mm} \\
\bottomrule
\\
\end{tabular}
\end{table}
\noindent

\section{Graphical Visualization}
To evaluate the approximation behavior of the Max--Min exponential sampling operator and its Kantorovich-type variant, we consider two distinct test functions over the interval \([0,2]\). 
\noindent
\subsection{Piecewise Function}
The first function is piecewise-defined:
\[
f(x) =
\begin{cases}
0.25 + 0.1x, & 0 \le x \le 0.4,\\[6pt]
0.85 - 0.05\sin(5x), & 0.4 < x \le 0.75,\\[6pt]
0.4 + 0.1x^2, & 0.75 < x \le 1.25,\\[6pt]
0.65 + 0.02\cos(3x), & 1.25 < x \le 2.
\end{cases}
\]
And approximated using the ramp kernel within the Max--Min and Kantorovich Max--Min exponential sampling frameworks for various sampling densities \(n \in \{10, 26,50, 100\}\). 
We denote by $\mathfrak{GM}_{n,\,\sigma_{R}}$  is the operator$\mathfrak{GM}_{n}$ defined using ramp kernel $ \sigma_{R}$ and by $\mathfrak{MK}_{n,\,\sigma_{R}}^{(m)}$ is the operator $\mathfrak{MK}_{n}^{(m)}$ with the same kernel $ \sigma_{R}.$
\begin{figure}[h!]
  \centering
  \includegraphics[width=0.8\linewidth]{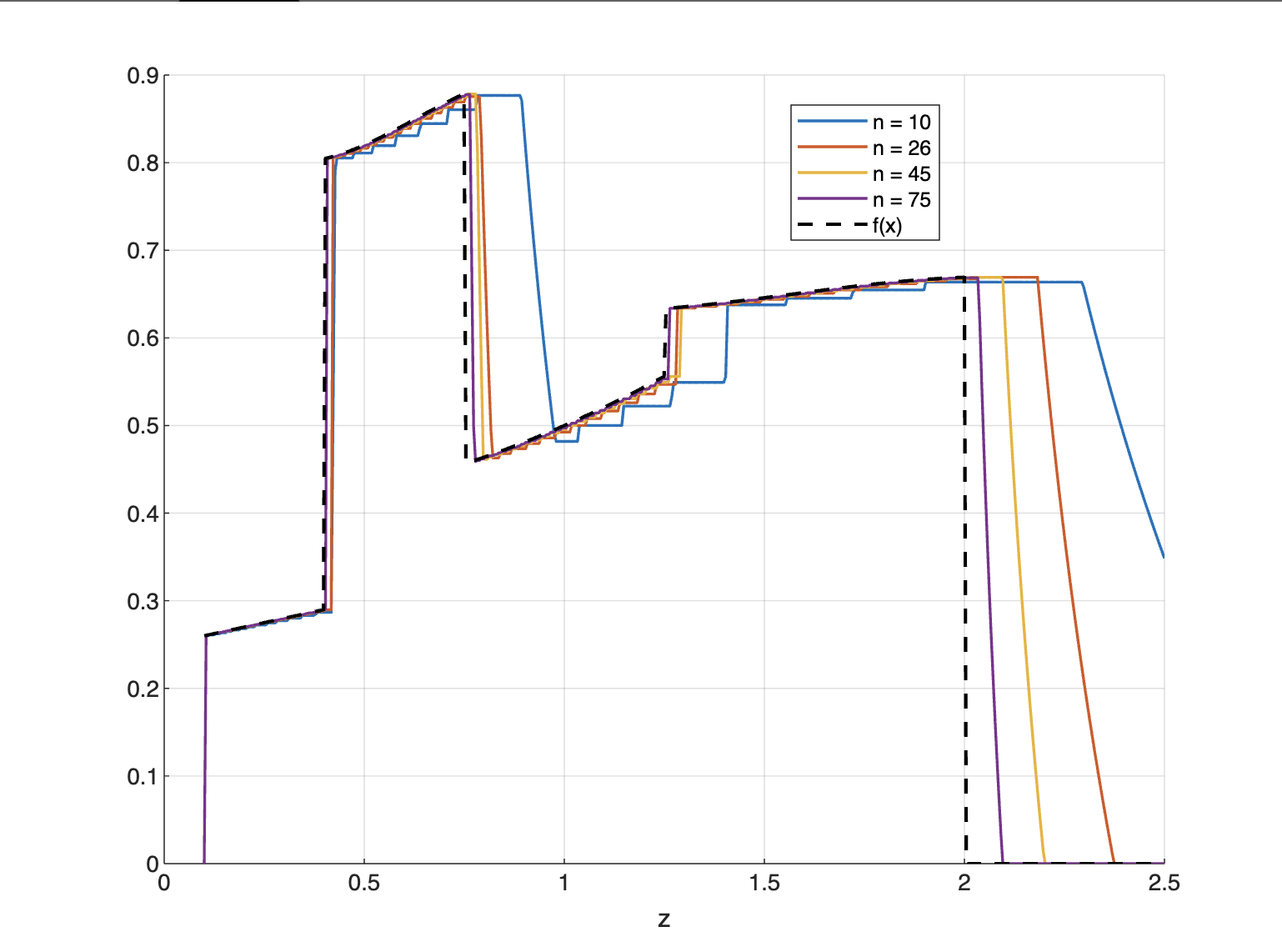}
  \caption{Approximation of $f(x)$ by $\mathfrak{GM}_{n,\sigma_{R}}(f,x)$ for $n=10,26,45,75.$}
  \label{fig:1}
\end{figure}

\vspace{0.1em}

\begin{figure}[h!]
  \centering
  \includegraphics[width=0.8\linewidth]{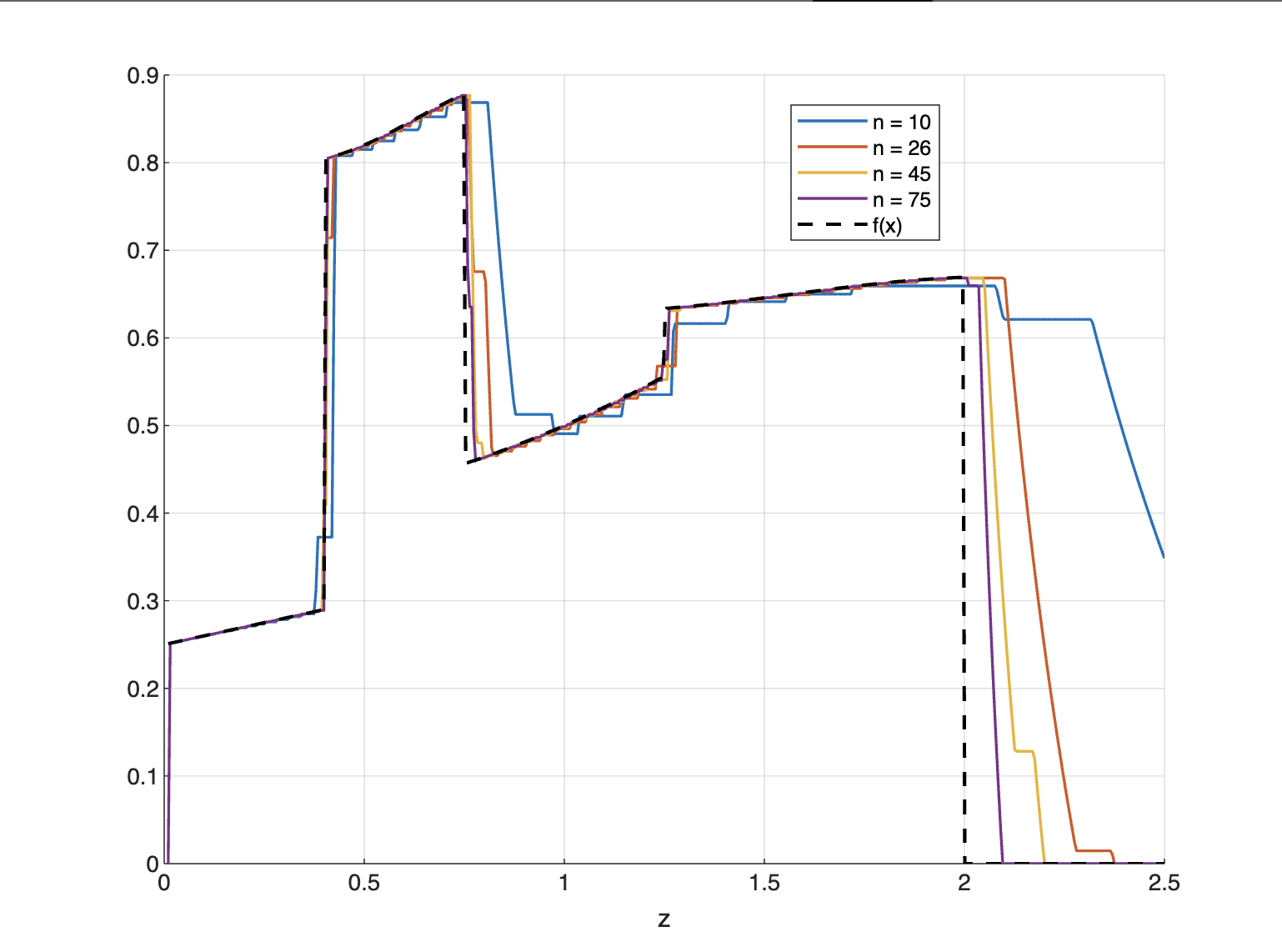}
  \caption{Approximation of $f(x)$ by $\mathfrak{MK}_{n,\sigma_{R}}^{(m)}(f,x)$ for $n=10,26 ,45, 75.$}
  \label{fig:2}
\end{figure}

\vspace{0.1em}

\begin{figure}[h!]
\centering
\includegraphics[width=0.8\textwidth]{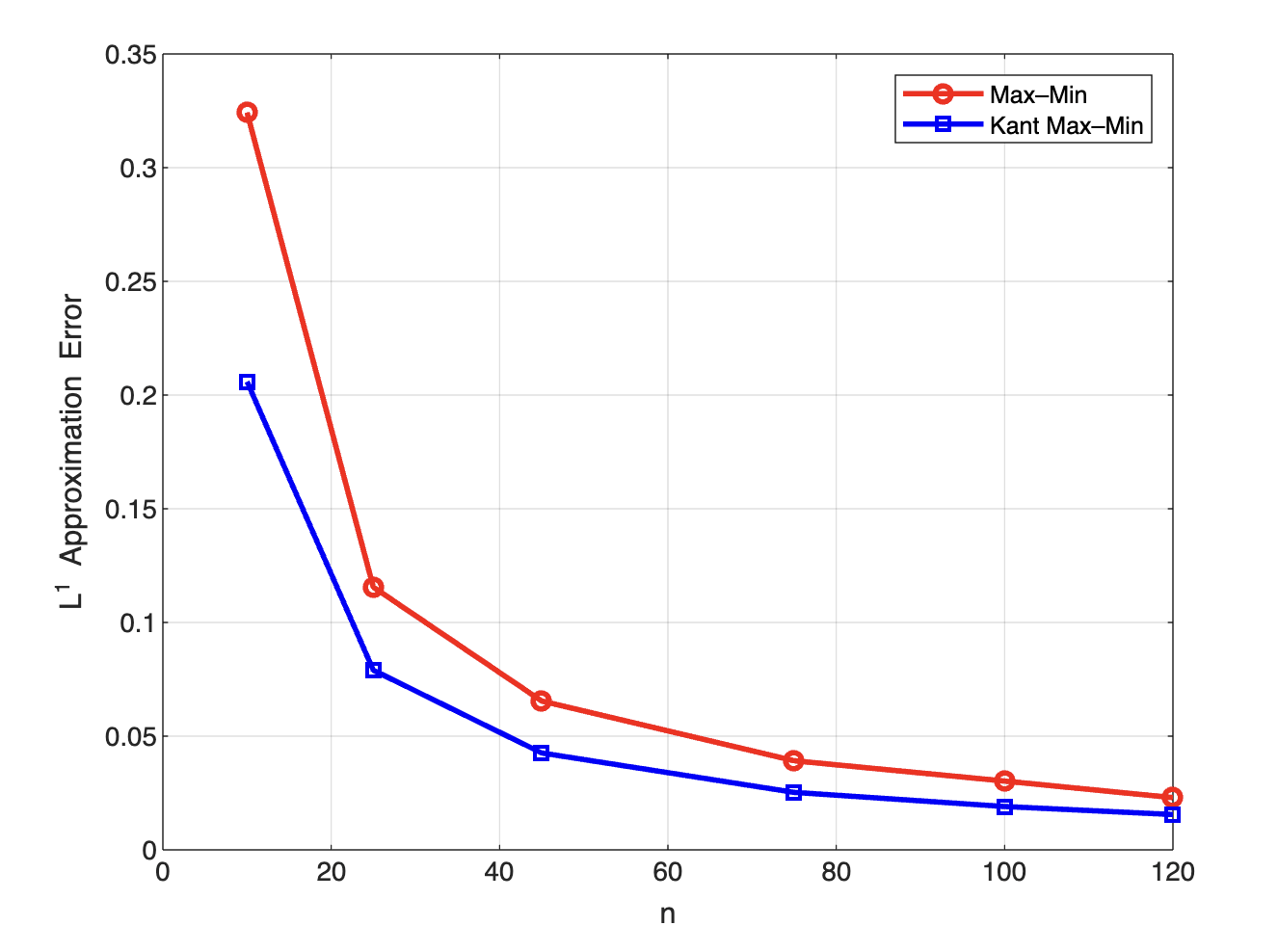}
\caption{Approximation error plot for $\mathscr{\mathfrak{GM}}_{n,\sigma_{R}}(f,z)$ and $\mathfrak{MK}_{n,\sigma_{R}}^{(m)}(f,z)$.}
\label{fig:3}
\end{figure}

\begin{table}[h!]
\centering

\begin{tabular}{|c|c|c|}
\hline
\textbf{n} & \textbf{ Max--Min Error} & \textbf{Kantorovich Max--Min Error} \\
\hline
10   & 0.324257 & 0.205913 \\
25   & 0.115541 & 0.079010 \\
45   & 0.065467 & 0.042613 \\
75   & 0.039184 & 0.025282 \\
100  & 0.030253 & 0.019063 \\
120  & 0.022967 & 0.015536 \\
\hline
\end{tabular}
\caption{Error comparison of L$^1$ approximation for the Max-Min and Kantorovich Max--Min exponential sampling operators using the ramp kernel for $f(x)$.}
\label{tab:2}
\end{table}

\subsection{Continuous Function}

The second function is globally smooth and oscillatory with decaying amplitude:
\[
g(x) = 0.2 \;+\; \frac{e^{\sin(x)} \,\sin(x^2)}{1 + x^4}, \quad x \in [0,2].
\] approximated using the ramp kernel within the Max--Min and Kantorovich Max--Min exponential sampling frameworks for various sampling densities \(n \in \{10, 26,50, 100\}\). 
\begin{figure}[h!]
  \centering
  \includegraphics[width=0.9\linewidth]{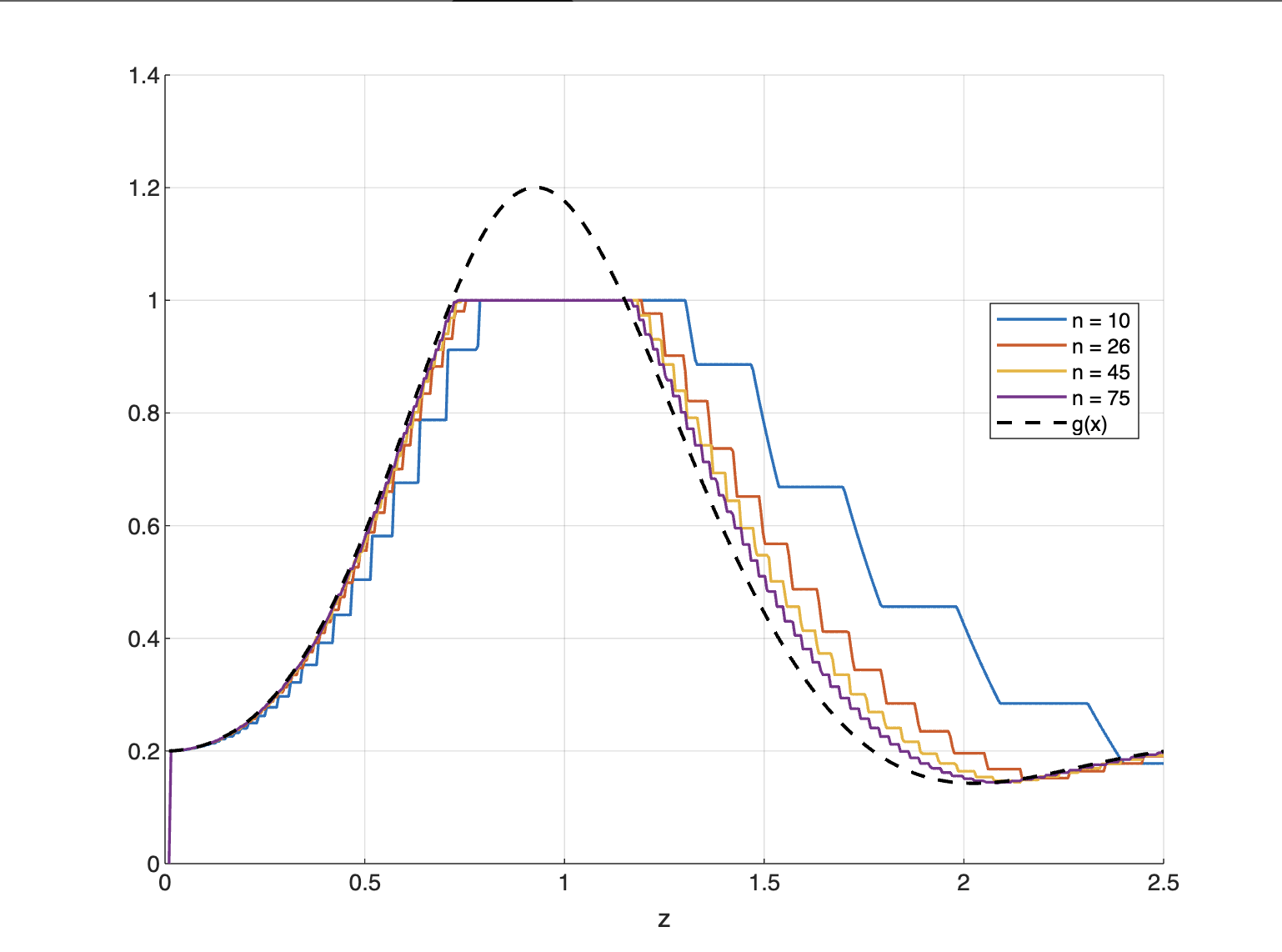}
  \caption{Approximation of $g(x)$ by $\mathscr{\mathfrak{GM}}_{n,\sigma_{R}}(g,x)$ for $n=10,26 ,45,75.$}
  \label{fig:4}
\end{figure}

\vspace{1em}

\begin{figure}[h!]
  \centering
  \includegraphics[width=0.9\linewidth]{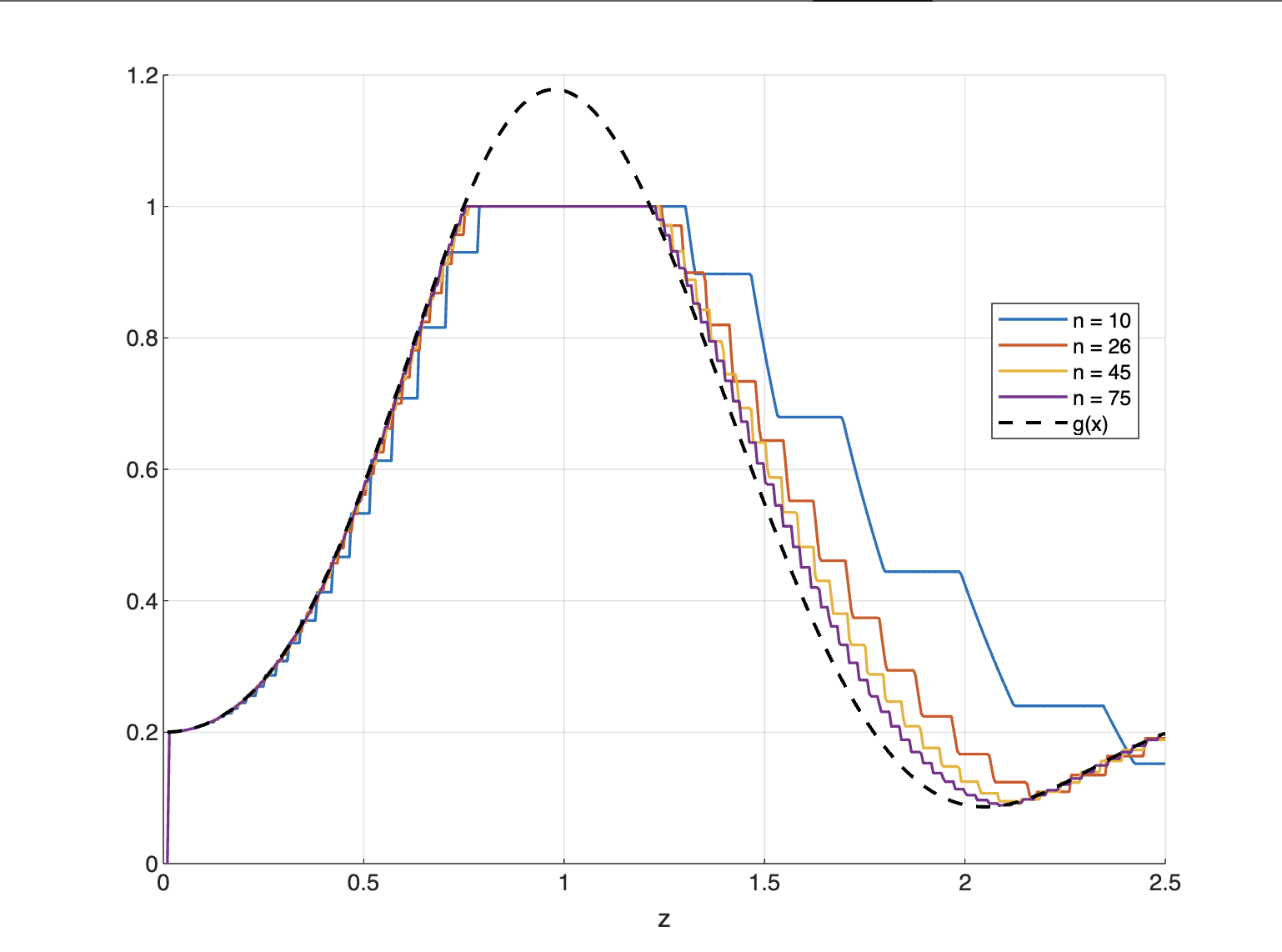}
  \caption{Approximation of $g(x)$ by $\mathscr{\mathfrak{MK}}_{n,\sigma_{R}}^{(m)}(g,x)$ for $n=10,26 ,45, 75.$}
  \label{fig:5}
\end{figure}

\vspace{1em}

\begin{figure}[h!]
\centering
\includegraphics[width=0.9\textwidth]{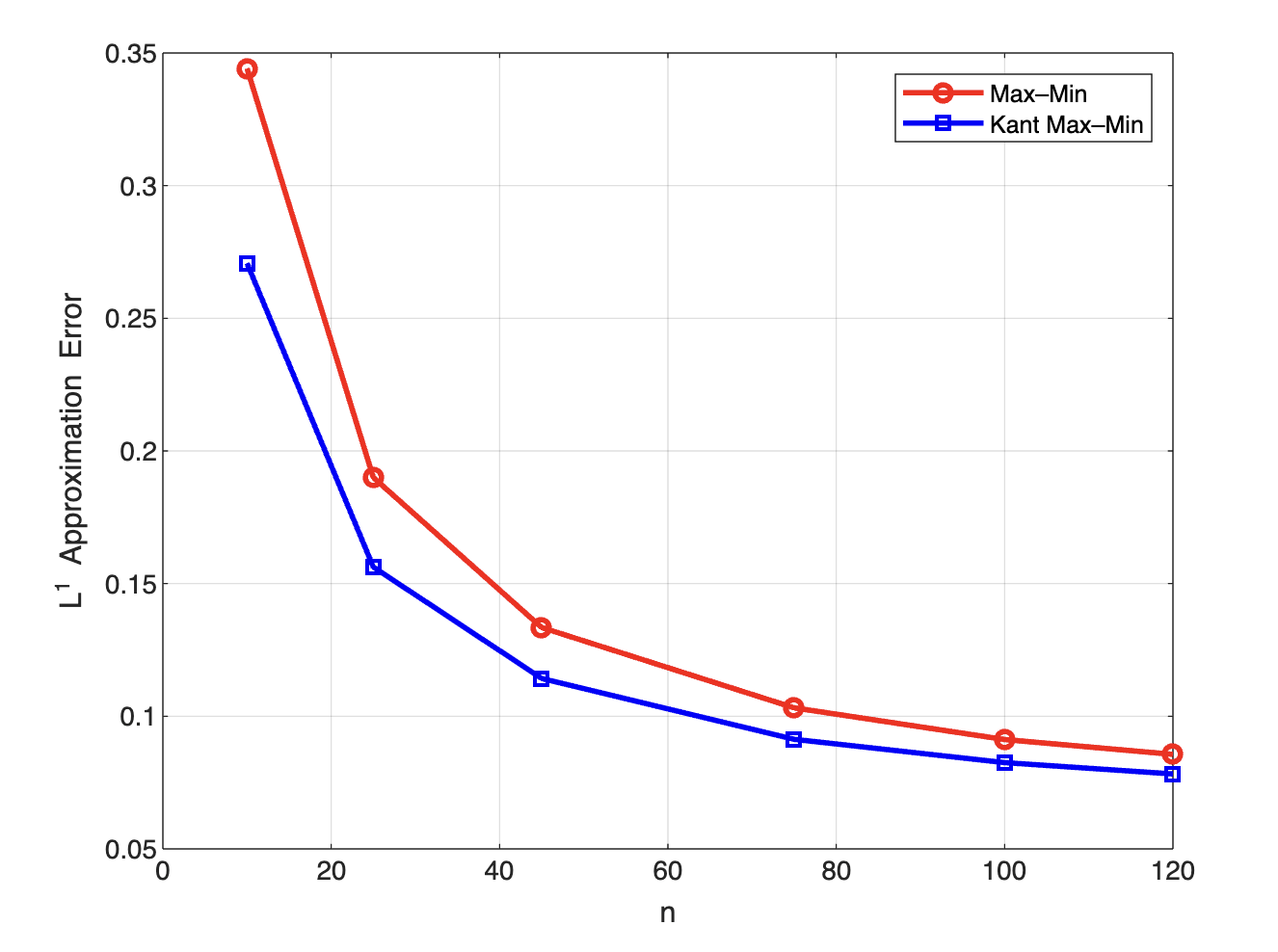}
\caption{\centering Approximation error plot for $\mathscr{\mathfrak{GM}}_{n,\sigma_{R}}(g,x)$ \,and \, $\mathscr{\mathfrak{MK}}_{n,\sigma_{R}}^{(m)}(g,x).$}
\label{fig:6}
\end{figure}

\begin{table}[h!]
\centering

\begin{tabular}{|c|c|c|}
\hline
\textbf{n} & \textbf{Max--Min Error} & \textbf{Kantorovich Max--Min Error} \\
\hline
10   & 0.344159 & 0.270741 \\
25   & 0.190002 & 0.156266 \\
45   & 0.133451 & 0.114245 \\
75   & 0.103169 & 0.091305 \\
100  & 0.091224 & 0.082526 \\
120  & 0.085628 & 0.078252 \\
\hline
\end{tabular}

\caption{Error comparison of L$^1$ approximation for the Max-Min and Kantorovich Max--Min exponential sampling operators using the ramp kernel for $g(x)$.}
\label{tab:maxmin_kant}
\end{table}

\newpage

\section{Conclusions}\label{section5}
In this paper, we studied the approximation properties of the Max-Min Kantrovich and classical Max-Min type  Exponential Sampling  neural network operator for functions in $\mathcal{U}_{b}(\mathscr{I}) \text{\,of log uniformly continuous and bounded function }\,\And\,\text{Mellin Orcliz space} \,\,\mathcal{L}_{\mu}^{\eta} $. We have proved the well-definedness of the operator and discussed its convergence, both in pointwise and uniform sense. The logarithmic modulus of continuity has been used to estimate the rate of convergence. We have illustrated the example of kernel functions and demonstrated the approximation of f(x) and g(x) for the proposed operator with ramp kernel. In figure~\ref{fig:3}   the comparison of approximation error value between  $\mathscr{\mathfrak{GM}}_{n,\sigma_{R}}$ and $\mathscr{\mathfrak{MK}}_{n,\sigma_{R}}^{(m)}$  has been illustrated and in table \ref{tab:2} approximation error has been enumerated with allocation of different values of n. The same approach has also been adopted for g(x)(see fig~\ref{fig:6}, and table~\ref{tab:maxmin_kant}). It is thus evident that the Kantorovich variant, owing to its integral formulation, better captures average behavior and mitigates oscillations near discontinuities. Overall, the synergy between the theoretical properties and the empirical evidence affirms that the proposed operators, especially when combined with the ramp kernel, offer a robust and adaptable framework for approximating both piecewise-defined and smoothly oscillatory functions in nonlinear exponential sampling contexts.



\subsection*{Conflicts of Interest}
	There is no conflict of interest regarding the publication of this article



\end{document}